\newif\iffinal 
\finaltrue 

\newif\ifIR
\IRtrue

\iffinal
  \newcommand{\hans}[1]{}
  \newcommand{\nico}[1]{}
  \newcommand{\michi}[1]{}
  \newcommand{\martin}[1]{}
  \newcommand{\chris}[1]{}
  \newcommand{\ph}[1]{}
  \newcommand{\yell}[1]{}
  \newcommand{\davidD}[1]{}
\else
  \newcommand{\hans}[1]{{\color{blue} \textbf{[Hans: #1]}}}
  \newcommand{\nico}[1]{{\color{orange} \textbf{[Nico: #1]}}}
  \newcommand{\michi}[1]{{\color{purple} \textbf{[Michi: #1]}}}
  \newcommand{\martin}[1]{{\color{olive} \textbf{[Martin: #1]}}}
  \newcommand{\chris}[1]{{\color{cyan} \textbf{[Chris: #1]}}}
  \newcommand{\ph}[1]{{\color{green} \textbf{[Phil: #1]}}}
  \newcommand{\yell}[1]{{\color{red} \MakeUppercase{\textbf{[#1]}}}}
  \newcommand{\davidD}[1]{{\color{cyan} \textbf{[David Duvenaud: #1]}}}
\fi
\ifIR
  \newcommand{\hansIR}[1]{}
  \newcommand{\nicoIR}[1]{}
  \newcommand{\michiIR}[1]{}
  \newcommand{\martinIR}[1]{}
  \newcommand{\chrisIR}[1]{}
  \newcommand{\phIR}[1]{}
  \newcommand{\yellIR}[1]{}
\else
  \newcommand{\hansIR}[1]{{\color{blue} \textbf{[Hans: #1]}}}
  \newcommand{\nicoIR}[1]{{\color{orange} \textbf{[Nico: #1]}}}
  \newcommand{\michiIR}[1]{{\color{purple} \textbf{[Michi: #1]}}}
  \newcommand{\martinIR}[1]{{\color{olive} \textbf{[Martin: #1]}}}
  \newcommand{\chrisIR}[1]{{\color{cyan} \textbf{[Chris: #1]}}}
  \newcommand{\phIR}[1]{{\color{green} \textbf{[Phil: #1]}}}
  \newcommand{\yellIR}[1]{{\color{red} \MakeUppercase{\textbf{[#1]}}}}
\fi

\documentclass{article}

\usepackage{microtype}
\usepackage{mathtools}
\usepackage{graphicx}
\usepackage{booktabs} 

\usepackage[accepted]{icml2020}

\usepackage{xcolor}
\usepackage[caption=false]{subfig}
\usepackage{pgfplots}
\pgfplotsset{compat=1.12}
\usepgfplotslibrary{groupplots}

\definecolor{shadecolor}{RGB}{230,230,250}
\definecolor{plotblue}{RGB}{57, 106, 177}
\definecolor{plotorange}{RGB}{218, 124, 48}
\definecolor{plotgreen}{RGB}{62, 150, 81}
\definecolor{plotred}{RGB}{204, 37, 41}
\definecolor{plotgray}{RGB}{83, 81, 84}
\definecolor{plotpurple}{RGB}{107, 76, 154}
\definecolor{plotdarkred}{RGB}{146, 36, 40}
\definecolor{plotbeige}{RGB}{148, 139, 61}
\definecolor{plotnavy}{HTML}{001F3F}

\usepackage{tikz}
\usepgfplotslibrary{external}
\usepgfplotslibrary{statistics}
\usetikzlibrary{decorations}
\usetikzlibrary{decorations.markings}
\usetikzlibrary{pgfplots.statistics}
\pgfplotsset{compat = 1.3}

\pgfplotscreateplotcyclelist{mycolorlist}{%
plotblue,every mark/.append style={fill=plotblue, draw = plotblue!35!black, thin},mark=*, mark size=1.2pt\\%
plotorange,every mark/.append style={fill=plotorange, draw = plotorange!35!black, thin},mark=square*, mark size=1.2pt\\%
plotgreen,every mark/.append style={fill=plotgreen, draw = plotgreen!35!black, thin},mark=diamond*, mark size=1.6pt\\%
plotred, every mark/.append style={fill=plotred, draw = plotred!35!black, thin},mark=triangle*, mark size=1.2pt\\%
plotgray, densely dotted, every mark/.append style={fill=plotgray, draw = plotgray!35!black, thin},mark=*, mark size = 1.2pt\\%
plotpurple, densely dotted,every mark/.append style={fill=plotpurple, draw = plotpurple!35!black, thin},mark=square*\\%
plotdarkred, densely dotted,every mark/.append style={fill=plotdarkred, draw = plotdarkred!35!black, thin},mark=diamond*, mark size=1.2pt\\%
plotbeige, densely dotted,every mark/.append style={fill=plotbeige, draw = plotbeige!35!black, thin},mark=triangle*, mark size=1.2pt\\%
}
\pgfplotsset{cycle list name = mycolorlist}
\pgfplotsset{every axis plot/.append style={very thick}}
\pgfplotsset{grid = major, major grid style = {thin}}
\pgfplotsset{ticks = major, tick align=inside, tick style = {thin, gray}}
\pgfplotsset{compat=1.11}
\pgfplotsset{width=\columnwidth, height=0.6\columnwidth}
\pgfplotsset{legend style={at={(1,0)},xshift=-0.05cm, yshift=0.05cm, anchor=south east }}
\pgfplotsset{
  colormap={whitered}{color(0cm)=(plotorange!75!plotred); color(1cm)=(white)},
}
\usepackage[eulergreek]{sansmath}

\pgfplotsset{
tick label style = {font=\sansmath\sffamily\tiny},
}

\usetikzlibrary{arrows.meta}

\usepackage{hyperref}

\usepackage{xr}
\usepackage{enumitem}

\usepackage{amsmath,amssymb,amsthm}
\usepackage{tensor}
\usepackage{xcolor}
\usepackage[nameinlink]{cleveref}
\usepackage[toc,page,titletoc]{appendix}
\usepackage{bbm}
\usepackage{mathtools}

\newcommand{\R}{\mathbb{R}}

\newcommand{\N}{\mathbb{N}}

\let\given\givenbase

\newcommand{\de}{\partial}  

\newcommand*{\defeq}{\coloneqq}

\newcommand*{\BO}{\mathcal{O}}
\newcommand*{\rd}{\mathrm{d}}

\newcommand*{\bs}{\boldsymbol}

\newtheorem{theorem}{Theorem}

\newtheorem{lemma}[theorem]{Lemma}
\newtheorem{assumption}{Assumption}

\newcommand{\dk}{\tensor[^{\de}]{k}{}}
\newcommand{\kd}{\tensor{k}{^{\de}}}
\newcommand{\dkd}{\tensor[^{\de}]{k}{^{\de}}}

\newcommand{\dKd}{\tensor[^{\de}]{K}{^{\de}}}
\newcommand{\dKid}{\tensor[^{\de}]{K}{^{\de}}(h:t_i)}
\newcommand{\dKTd}{\tensor[^{\de}]{K}{^{\de}}(h:T)}

\newcommand\SmallbMatrix[1]{{%
  \tiny\arraycolsep=0.3\arraycolsep\ensuremath{\begin{bmatrix}#1\end{bmatrix}}}}

\newcommand\norm[1]{{
    \left \Vert #1 \right \Vert
}}
\newcommand\absval[1]{{
    \left \vert #1 \right \vert
}}

\newcommand\transJac[1]{
\left [D #1 \right ]^{\intercal}
}

\newcommand{\smallddots}{\vphantom{\int\limits^x}\smash{\ddots}}

\DeclareMathOperator{\Tr}{trace}

\newcommand{\interior}[1]{%
  {\kern0pt#1}^{\mathrm{o}}%
}

\newlength{\figwidth}
\newlength{\figheight}

\crefname{supp}{Supplement}{Supplements}

\icmltitlerunning{Likelihoods for `Likelihood-Free' Dynamical Systems}

\begin{document}

\twocolumn[
\icmltitle{Differentiable Likelihoods for Fast Inversion \\
           of `Likelihood-Free' Dynamical Systems}



\icmlsetsymbol{primary}{*}

\begin{icmlauthorlist}
  \icmlauthor{Hans Kersting}{primary,ekut,mpi,bcai}
  \icmlauthor{Nicholas Kr{\"a}mer}{primary,ekut}\\
  \icmlauthor{Martin Schiegg}{bcai}
  \icmlauthor{Christian Daniel}{bcai}
  \icmlauthor{Michael Tiemann}{bcai}
  \icmlauthor{Philipp Hennig}{ekut,mpi}
\end{icmlauthorlist}

\icmlaffiliation{ekut}{University of T\"ubingen, T\"ubingen, Germany}
\icmlaffiliation{mpi}{Max Planck Institute for Intelligent Systems, T\"ubingen, Germany}
\icmlaffiliation{bcai}{Bosch Center for Artificial Intelligence, Renningen, Germany}

\icmlcorrespondingauthor{Hans Kersting}{hans.kersting@uni-tuebingen.de}
\icmlcorrespondingauthor{Nicholas Kr\"amer}{nicholas.kraemer@uni-tuebingen.de}

\icmlkeywords{Inverse problems, dynamical systems, probabilistic numerics}

\vskip 0.3in
]



\printAffiliationsAndNotice{$^*$Primary author} 

\begin{abstract}
Likelihood-free (a.k.a.~simulation-based) inference problems are inverse problems with expensive, or intractable, forward models. 
ODE inverse problems are commonly treated as likelihood-free, as their forward map has to be numerically approximated by an ODE solver. 
This, however, is not a fundamental constraint but just a lack of functionality in classic ODE solvers, which do not return a likelihood but a point estimate.
To address this shortcoming, we employ Gaussian ODE filtering (a probabilistic numerical method for ODEs) to construct a local Gaussian approximation to the likelihood.
This approximation yields tractable estimators for the gradient and Hessian of the (log-) likelihood.
Insertion of these estimators into existing gradient-based optimization and sampling methods engenders new solvers for ODE inverse problems. 
We demonstrate that these methods outperform standard likelihood-free approaches on three benchmark-systems.
\end{abstract}

\davidD{In the abstract, it's not clear if the likelihood is coming from uncertainty in the solve, or some other noise source.}
\chrisIR{Abstract und Intro: zielen mir persoenlich zu schnell auf technische Details, ohne dass ich das Gefuehl habe den Rahmen vollstaendig verstanden zu haben. 
In der Intro gibt es zB mehrere Forwardpointer auf Gleichungen des Hauptteils, das finde ich eher verwirrend. Gleichzeitig wird mir der konkrete Nutzen und die Usecases nicht deutlich genug.}

\section{Introduction}

\hansIR{Use the seminal paper \citet{DiggleGratton_MCMCofInferenceforImplicitStatisticalModels_1984} and \citet{cranmer_sbi_2019} in the introduction.}
\martinIR{The introduction should include more introductory notation. You could also give a little bit more context for readers who are not directly working in this field. See comments in highlighted phrases about the notation that should be introduced (either formally or by giving more context).}

Inferring the parameters of dynamical systems that are defined by ordinary differential equations (ODEs) is of importance in almost all areas of science and engineering.
Despite the wide range of available ODE inverse problem solvers, simple random-walk Metropolis methods remain the go-to solution; see e.g.~\citet[Section 2.4]{tarantola_invprob_2005}.
That is to say that ODE inverse problems are routinely treated as if their forward problems were black boxes.
The reason usually cited for this generic approach is that ODE forward solutions are highly non-linear and numerically intractable for all but the most trivial cases.
\hans{Maybe cite \citet{VyshemirskyGirolami2008} here? They argue this.}
Therefore, it is common to consider ODE inverse problems as `likelihood-free' inference (read: intractable likelihood)---a.k.a.~simulation-based inference or, in the Bayesian case, Approximate Bayesian Computation (ABC); see \citet{cranmer_sbi_2019} for an up-to-date examination of these closely-related areas.
\\
We here argue that, at least for ODEs, this approach is mistaken.
If a dynamical system is accurately described by an ODE, its explicit mathematical definition should be exploited to design efficient algorithms---not ignored and treated as a black-box, likelihood-free inference problem.
\\
To this end, we construct a local Gaussian approximation of the likelihood by Gaussian ODE Filtering, a probabilistic numerical method (PNM) for ODE forward problems. 
(\Cref{suppl:introduction_ODE_Filtering} provides a concise introduction to Gaussian ODE filtering; \citet{TronarpKSH2019} offer a more detailed presentation. See \citet{HenOsbGirRSPA2015} or \citet{OatesSullivan2019} for a broad introduction to PNMs.)
\begin{figure}[t]
\setlength{\figwidth}{.9\columnwidth}
\setlength{\figheight}{.5\figwidth}
\centering\scriptsize
\includegraphics{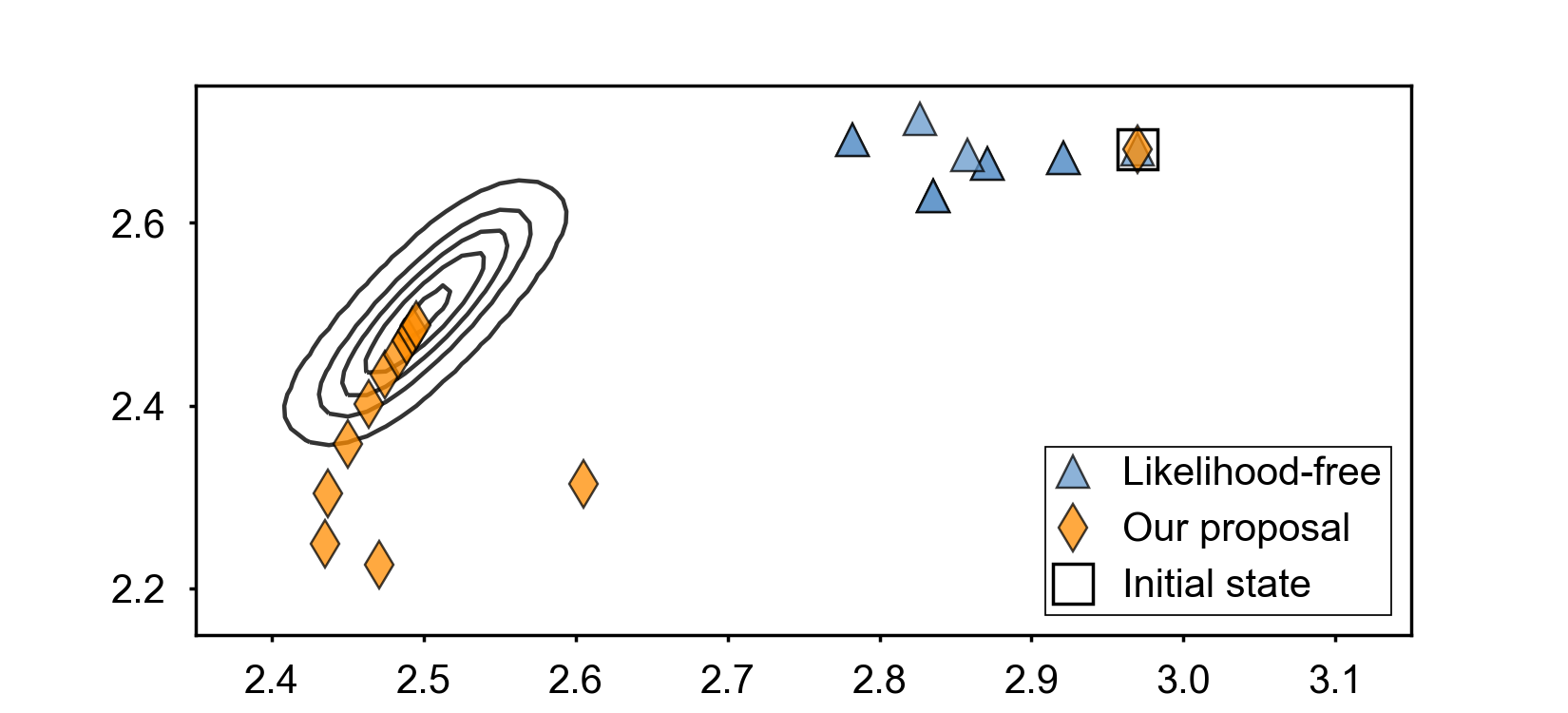}
\caption{Inference on the logistic ODE. First twelve sampled parameters of
likelihood-free inference and our proposed method. Details in text. \label{figure:our_proposal_sampling}
\davidD{What are the contours?}
}
\end{figure}
The key insight of our work is that there \emph{is} a likelihood in simulations of ODEs, and in fact it can be approximated cheaply, and analytically: The mean estimate $\bs{m}_{\theta}$ of the forward solution computed by Gaussian ODE filters can be linearized in the parameter $\theta$, so that gradient, Hessian, etc.~of the approximated log-likelihood can---via a cheap estimator $J$ of the Jacobian of the map $\theta \mapsto \bs{m}_{\theta}$---be computed in closed form (\Cref{sec:gradient_and_hessian_estimators}).
In this way, the probabilistic information from Gaussian ODE filtering yields a tractable, twice-differentiable likelihood for `likelihood-free' ODE inverse problems. This enables the use of first and second-order optimization or sampling methods (see \Cref{figure:our_proposal_sampling}).
\\
Much thought has been devoted to improving the slow run-times of ODE inverse inference---which is due to the laborious explicit numerical integration per parameter.
In machine learning, e.g.,~authors have proposed to reduce the amount of necessary parameters by active learning with Gaussian process (GP) surrogate likelihoods \citep{meeds_welling_GPSABC}, or even to avoid numerical integration altogether by gradient matching \citep{calderhead2008accelerating}.
This paper adds a new way to reduce the amount of parameters by employing gradient (and Hessian) estimates of the log-likelihood. 

{\bf Contributions} \ 
The main contributions are twofold: 
\emph{Firstly}, we introduce tractable estimators for the gradients and Hessian matrices of the log-likelihood of ODE inverse problems by Gaussian ODE filtering.
To derive these estimators, we construct a new estimator $J$ for the Jacobian of the forward map.
We theoretically support the use of $J$ by a decomposition of the true Jacobian into $J$ and a sensitivity term $S$ (see \Cref{theorem:true_Jacobian_of_mtheta}), as well as an upper bound on its approximation error (see \Cref{theorem:bound_on_approximation_error}).
\emph{Secondly}, we propose a range of new solvers which require gradients and/or Hessians, by inserting these estimators into first and second-order optimization and sampling methods.
The utility of these algorithms is demonstrated by experiments on three benchmark ODEs where they outperform their gradient-free counterparts.

\section{Problem Setting}  \label{sec:ProblemSetting}

\martinIR{This equation (without assumption 1) should maybe go into the introduction. Then you can reuse the equation to introduce the notation/context.}
We consider a dynamical system defined by the ODE
\begin{align} \label{eq:ODE}
  \dot x(t) 
  = 
  f \left(x(t), \theta \right ), 
  \qquad 
  x(0)=x_0 \in \R^d,
\end{align}
on the finite time domain $t \in [0, T]$ for some $T>0$, with parametrized vector field $f : \R^d \times \R^n \to \R^d $.
We restrict our attention to choices of $f$ satisfying the following
\begin{assumption}  \label{ass:f_linear_in_theta}
  $f(x,\theta) = \sum_{i=1}^n \theta_i f_i(x)$, for some continuously differentiable $f_i: \R^d \to \R^d$, for all $i = 1,\dots,n$. 
\end{assumption}
\hans{This implies existence of $x$ on $[0,T]$ by the Picard--Lindel\"of theorem as given by \citep[Theorem 8.13]{KelleyPeterson2010TheoryOfODEs}.}
The necessity for this assumption will become evident in \Cref{subsubsection:TheFilteringMean}.
It is not very restrictive: e.g.~the corresponding assumption in \citet[eq. (10)]{GorbachBauerBuhmann17} is stronger.
In fact, most standard ODEs collected in \citet[Appendix I]{hull1972comparing}, a standard set of ODE benchmarking problems, satisfy \Cref{ass:f_linear_in_theta} either immediately or after reparametrization.
Otherwise, we can still transform a non-conforming ODE into a system that obeys \Cref{ass:f_linear_in_theta}, as exemplified for the protein signalling transduction pathway in \Cref{subsubsec:exp_pst}.
While this adds an additional layer of imprecision, the experiments appear to be equally good---which suggests a wider applicability of our methods than \Cref{ass:f_linear_in_theta}.

If the initial value $x_0$ is unknown too (as is often the case in practice), it can be treated as a parameter by defining a new parameter vector $\begin{pmatrix} x_0^{\intercal},\theta^{\intercal}\end{pmatrix}^{\intercal} \in \R^{d+n}$; see \cref{eq:mean_linear_in_param}.
Solving \cref{eq:ODE}, for a given $\theta$, with a numerical method is known as the \emph{forward problem}.
\\
For the \emph{inverse problem}, we assume the dynamical system described by \cref{eq:ODE} with \emph{unknown} true parameter $\theta^{\ast}$. The true trajectory $x = x_{\theta^{\ast}}$ is observed under additive, zero-mean Gaussian noise at $M$ discrete times $0 \leq t_1 < \dots < t_M \leq T$:
\begin{align}   \label{eq:def_data}
  z(t_i) 
  \defeq 
  x(t_i) + \varepsilon_i \in \R^d,
  \quad
  \varepsilon_i \sim \mathcal N(0,\Sigma_i),
\end{align} 
for all $i \in \{1,\dots,M\}$.
Below we assume, w.l.o.g., that $\Sigma_i = \Sigma$, for all $i \in \{1,\dots,M\}$.
We define the stacked data across $M$ time points and $d$ dimensions as
\begin{align*}
  \bs{z} \defeq \begin{bmatrix} z_1(t_1),\allowbreak \dots, z_1(t_M),\allowbreak \dots, z_d(t_1),\allowbreak \dots, z_d(t_M) \end{bmatrix}^{\intercal},
\end{align*}
and analogously, for all $\theta \in \Theta$, the true solution at these points as $\bs{x}_{\theta}$.
The inverse problem consists of inferring the parameter $\theta^{\ast}$ that generated the data through \cref{eq:def_data}.
For the sake of readability, we will assume w.l.o.g.~that $d=1$; this restriction is purely notational as can be seen from the multi-dimensional experiments below.
Under these conventions, \cref{eq:def_data} is equivalent to 
\begin{align}   \label{eq:p_z_given_x}
  p(\bs{z} \given \bs{x})
  =
  \mathcal{N}\left(\bs{z};  \bs{x}, \sigma^2 I_M \right)
\end{align}
for some $\sigma^2 > 0$, where $I_M$ is the $M\times M$ identity matrix.
Heteroscedastic noise can be modelled by replacing $\sigma^2 I_M$ with a diagonal matrix with varying diagonal entries.


\section{Likelihoods by Gaussian ODE Filtering}   \label{sec:likelihoods_by_GODEF}
\chrisIR{Problem Setting \& Method: Gefuehlt gibt es relativ viele Nebenbaustellen oder Seitenargumente, die mich teilweise etwas von der Hauptstoryline abgelenkt haben. Vllt waere es eine Ueberlegung Wert die Annahmen und die eigentliche Methodenbeschreibung staerker zu trennen um eine moeglichst kompakte und eingaengige Uebersicht des vorgeschlagenen Ansatzes zu erreichen. }
\phIR{Stimmt! Alternativ koennte auch am Anfange (Section 1) eine konzeptionelle Einfuehrung stehen.}
\hansIR{@Philipp: Ich bin mir nicht sicher, ob ich deinen Vorschlag richtig verstehe. Meinst du, dass wir zunaechst eine konzeptionelle und dann eine technische Einleitung schreiben?}
\martinIR{I agree with Chris and Philipp. Sections 1-3 (until excluding 3.1) should be restructured for a more fluent story. Proposal:
- Show equation 1 and explain the high-level task.
- With the help of equation 1 introduce the necessary notations ("forward problem", "inverse problem",...)
- Talk about challenges and shortcomings of the previous methods.
- Briefly summarize your method (high-level without math).
- Collect the contributions and state why they address the challenges/shortcomings of previous approaches.
Then in a second part, you could be a bit more specific and introduce a bit more math.} 
The prevailing view on the uncertainty in inverse problems only considers the aleatoric uncertainty $\Sigma_i$ from \cref{eq:def_data} and ignores the epistemic uncertainty over the quality of the employed numerical approximation $\hat{x}_{\theta}$ of $x_{\theta}$.
In other words, the likelihood of the forward problem, $p(\bs{x}_{\theta} \given \theta)$, is commonly treated as a Dirac distribution $\delta(\bs{x}_{\theta} - \bs{\hat x}_{\theta} )$ which yields the \emph{uncertainty-unaware likelihood}
\begin{align}
  p(\bs{z}\given \theta) 
  &= 
  \int p(\bs{z}\given \bs{x}_{\theta}) p(\bs{x}_{\theta} \given \theta) \ \rd \bs{x}_{\theta}
  \label{subeq_a:likelihood}
  \\
  &=
  \int p(\bs{z}\given \bs{x}_{\theta}) \delta(\bs{x}_{\theta} - \bs{\hat x}_{\theta} ) \ \rd \bs{x}_{\theta}
  \\
  &\stackrel{\cref{eq:p_z_given_x}}{=}
  \mathcal{N}\left(\bs{z}; \bs{\hat x}_{\theta}, \sigma^2 I_M \right).
  \label{subeq_b:likelihood}
\end{align}
as the `true' intractable likelihood.
\davidD{That likelihood also represents uncertainty, just not from one specific source, which is numerical error.}
\davidD{I would frame your method as augmenting the standard likelihood with extra uncertainty due to numerical error.  Throughout the first two sections it's a bit unclear if and when there is already a likelihood.  Also, I think your method would still apply even if the original noise likelihood wasn't Gaussian.}
This, however, ignores the epistemic uncertainty over the accuracy $\hat x_{\theta}$ which leads to overconfidence.
This uncertainty is due to the discretization error of the numerical solver used to compute $\hat x_{\theta}$, and can only be avoided for the most trivial ODEs.
This problem has previously been recognized in, e.g., \citet[Section 3.2]{conrad_probability_2017} and \citet[Section 8]{AbdulleGaregnani17} who, as a remedy, construct a `cloud' of possible solutions by running a classical solver multiple times with a prespecified accuracy. 
This, unfortunately, requires the computational invest of several forward solves for the same $\theta$, which could instead be used for additional $\theta$, or higher accuracy.
\\
To obtain such uncertainty quantification more cheaply, we employ Gaussian ODE filtering with a once-integrated Brownian motion (IBM) prior on $x$; see \Cref{suppl:subsec:GaussianODEFiltering} for a short introduction.
This amounts---e.g.~in the notation of \citet{TronarpKSH2019}---to setting $q=1$. 
Gaussian ODE filtering has the advantage over other numerical solvers, probabilistic or classical, that we can compute gradients of the likelihood, as demonstrated below.
For a given $\theta$, the Gaussian ODE filter computes a multivariate normal distribution over $x_{\theta}$ at a set of $N=T/h$, for notational simplicity, equidistant time points $\{0, h, \dots, Nh\}$ with step size $h>0$.
This set is, w.l.o.g., assumed to contain the data time points $\{t_1,\dots,t_M\}$ from \cref{eq:def_data}, i.e.~we assume the existence of a set of integers $\{l_1,\dots,l_M\}$ such that $t_i = l_i h$. (The  w.l.o.g.~assumption can otherwise be satisfied by interpolating along the dynamic model; see \cref{eq:dynamic_model} in \Cref{suppl:introduction_ODE_Filtering}.)
\subsection{The Filtering Distribution}
The Gaussian ODE filter returns the so-called (posterior) filtering distribution over the ODE solution $\bs{x}_{\theta}$, given by
\begin{align}    \label{eq:filtering_distribution}
  p(\bs{x}_{\theta} \given \theta)
  =
  \mathcal{N}(\bs{x}_{\theta}; \bs{m}_{\theta}, \bs{P} ),
\end{align}
with $\bs{m}_{\theta} \in \R^M$ and $\bs{P} \in \R^{M \times M}$ given below by \cref{eq:mean_linear_in_param} and \cref{eq:filtering_cov}, respectively.
This probabilistic likelihood yields the new \emph{uncertainty-aware likelihood}
\begin{align}
    p(\bs{z}\given \theta)
    &=
    \int p(\bs{z}\given \bs{x}_{\theta}) \mathcal{N}(\bs{x}_{\theta}; \bs{m}_{\theta}, \bs{P} ) \ \rd \bs{x}_{\theta}
    \\
    &\stackrel{\mathclap{\cref{eq:p_z_given_x}~}}{=}
    \mathcal{N}(\bs{z}; \bs{m}_{\theta}, \bs{P} + \sigma^2 I_M )
    \label{subeq:uncertainty_aware_likelihood}
\end{align}
which has two advantages over the uncertainty-unaware likelihood from \cref{subeq_b:likelihood}:
\begin{enumerate}
    \item The filtering mean $\bs{m}_{\theta}$ can be linearized in $\theta$, as specified below in \cref{eq:mean_linear_in_param}. This yields an estimate $J$ of the Jacobian matrix of $\theta \mapsto \bs{m}_\theta$ which implies estimators of gradients and Hessian matrices of the likelihood; see \cref{eq:gradient_estimator_given_by_J,eq:Hessian_estimator_given_by_J}.
    These estimators are useful to guide samples of $\theta$ into regions of high likelihood by the gradient-based sampling and methods defined in \Cref{sec:novel_gradientbased_algorithms} below.
    \item The variance $\bs{P}$ captures the average-case squared (epistemic) error $\norm{\bs{m}_{\theta} - \bs{x}_{\theta}}^2$, and can be added to the (aleatoric) variance $\Sigma_i$; see \cref{subeq:uncertainty_aware_likelihood}. Unless $\bs{P} \ll \sigma^2 I_M$, this prevents over-confidence, as visualized in \Cref{fig:uncertainty-(un)awareness}.
\end{enumerate}
\begin{figure}[ht]
\centering\scriptsize
\includegraphics{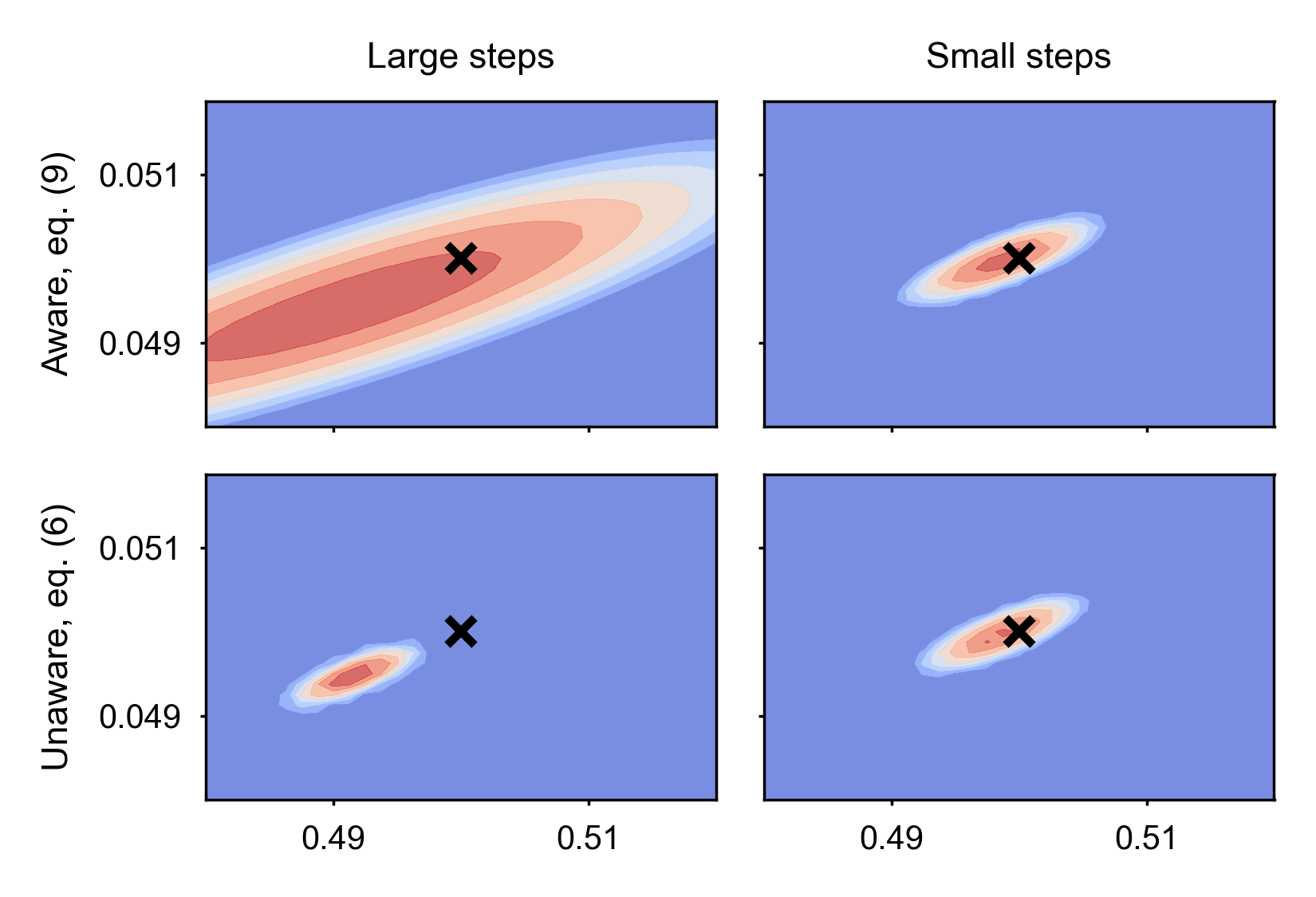}
\caption{Uncertainty-(un)aware likelihoods, \cref{subeq_b:likelihood,subeq:uncertainty_aware_likelihood} w.r.t.~$(\theta_1, \theta_2)$ of Lotka-Volterra ODE, \cref{eq:lotka_volterra}, with fixed $(\theta_3, \theta_4) = (0.05, 0.5)$. $\theta_1$ on $x$ and $\theta_2$ on $y$-axis. Black cross is true parameter.
The unaware likelihood is overconfident for the large step size ($h=0.2$), i.e.~for large $\bs{P}$, while the aware likelihood has calibrated uncertainty.
For the small step size ($h=0.025$) this effect is less pronounced as $\bs{P}$ is small.
\hansIR{Cf.~\citet[Fig.~9]{AbdulleGaregnani17} and \citet[Fig.~6]{conrad_probability_2017}.}
\label{fig:uncertainty-(un)awareness}
}
\end{figure}
\hans{Move this plot to the bottom?}
In the following two subsections, we provide explicit formulas for $\bs{m}_{\theta}$ and $\bs{P}$.
A detailed derivation of these formulas is given in \Cref{suppl:EquivForm}.

\subsubsection{The Filtering Mean}   \label{subsubsection:TheFilteringMean}

Under \Cref{ass:f_linear_in_theta}, the filtering mean $\bs{m}_{\theta} = [m_{\theta}(t_1), \dots, m_{\theta}(t_M) ]^{\intercal}$ is given by
\begin{align}    \label{eq:mean_linear_in_param}
    \bs{m}_{\theta}
    =
    \begin{bmatrix} \mathbbm{1}_M & J \end{bmatrix} \begin{bmatrix} x_0 \\ \theta  \end{bmatrix}
    =
    x_0 \cdot \mathbbm{1}_M  + J \theta
    \
    \in \R^{M}
    ,
\end{align}
where $\mathbbm{1}_M= [1,\dots,1]^{\intercal}$ denotes a vector of $M$ ones.
Hence, $\bs{m}_{\theta}$ is linear in $\theta$ as well as in the extended parameter vector $[x_0,\theta^{\intercal}]^{\intercal}$.
(A more detailed derivation of \cref{eq:mean_linear_in_param} is provided in \Cref{suppl:derivation_of_linear_mean}.)
Here,
\begin{align}    \label{eq:def_Jacobian}
    J
    \defeq
    K Y
    \
    \in \R^{M \times n}
\end{align}
is an estimator of the Jacobian matrix of the map $\theta \mapsto \bs{m}_{\theta}$, as we show in \Cref{theorem:true_Jacobian_of_mtheta} below.
This estimator is equal to the product of the \emph{kernel prefactor} $K$ and the \emph{evaluation factor} $Y$.
The kernel prefactor $K$ is given by
\begin{align}   \label{eq:def_K}
    K
    \defeq
    \begin{bmatrix}
        \kappa_1 , \dots , \kappa_M
    \end{bmatrix}^{\intercal}
    \
    \in \R^{M \times N}
    ,
\end{align}
whose $i$-th row is
\begin{align}   \label{eq:def_kappa}
    \kappa_i
    \defeq
    \begin{bmatrix}
        \tilde{\kappa}_i^{\intercal}
        ,
        0, \dots, 0  
    \end{bmatrix}^{\intercal}
    \
    \in \R^{N},
\end{align}
which is defined by 
\begin{align}  \label{eq:def_tilde_kappa}
    \tilde{\kappa}_i
    \defeq 
    \left [ \dKid  + R \cdot I_{l_i}  \right ]^{-1} \kd(h:t_i,t_i)
    \ 
    \in \R^{l_i},
\end{align}
for some measurement variance $R \geq 0$.
Here, $\kd = \de k(t,t^{\prime}) /  \de t^{\prime}$ and $\dkd = \de^2 k(t,t^{\prime}) / \de t \de t^{\prime} $ are derivatives of the IBM kernel $k$, and, analogously, the cross-covariance w.r.t.~the kernel $\dk$ and the kernel Gram matrix w.r.t.~the kernel $\dkd$ up to time $t_i$ are denoted by
\begin{align}
    \label{eq:def_kd}
    \kd(h:t_i, t_i)
    &\defeq
    \begin{bmatrix} \kd(t_i,h), \dots , \kd(t_i,t_i) \end{bmatrix}^{\intercal}, \ \text{and}
    \\
    \dKid
    &\defeq
    \SmallbMatrix{ \dkd(h,h) & \cdots & \dkd(l_i h, l_i h) \\ \vdots & \ddots & \vdots \\ \dkd(l_i h,h) & \cdots & \dkd(l_i h, l_i h) }.
    \label{eq:def_dKid}
\end{align}
Now, recall \Cref{ass:f_linear_in_theta}.
For a given $\theta$, the entries of the evaluation factor $Y \in \R^{N \times n}$ are
\begin{align}   \label{eq:def_data_factor}
    y_{ij}
    \defeq 
    f_j(m_{\theta}^-(ih)) - f_j(x_0),
\end{align}
for all $i=1,\dots,N$ and $j=1,\dots,n$, where $m_{\theta}^-(ih)$ is the predictive mean of the ODE Filter at $t=ih$.
Note that the Gaussian ODE Filter computes the $f_j(m_{\theta}^-(ih))$ and $f_j(x_0)$ for every forward solve as intermediate quantities, to evaluate the right-hand side of \cref{eq:ODE}. 
Hence, $Y$ is freely accessible with every filtering distribution, \cref{eq:filtering_distribution}.
However, as an estimate of $x_{\theta}(ih)$, $m_{\theta}^-(ih)$ depends on $\theta$ in a nonlinear and potentially sensitive way.
By ignoring this dependence in the above notation, we, strictly speaking, also omit the dependence of $Y$ and, thereby, $J$ on $\theta$ (more in \Cref{suppl:derivation_of_linear_mean}).
For this reason, $J$ is not the true Jacobian of $\theta \mapsto \bs{m}_{\theta}$ but only an estimator (see \Cref{subsec:JacobianEstimator}).
\subsubsection{The Filtering Covariance}   \label{subsubsec:FilteringCovariance}
The entries of the covariance matrix $\bs{P} \defeq \operatorname{diag}(P(t_1), \dots, P(t_M)) \in \R^{M \times M} $ of the filtering distribution from \cref{eq:filtering_distribution} coincide with the GP-posterior variances, i.e.
\begin{align}   \label{eq:filtering_cov}
    P(t_i)
    =
    &\SmallbMatrix{ k(h,h) & \cdots & k(l_i h, l_i h) \\ \vdots & \ddots & \vdots \\ k(l_i h,h) & \cdots & k(l_i h, l_i h) } 
    -
    \kd(h:t_i, t_i)^{\intercal}
    \notag
    \\
    &\times \left [ \dKid  + R \cdot I_l  \right ]^{-1}
    \kd(h:t_i, t_i),
\end{align}
and are hence independent of $\theta$. 
(See \Cref{suppl:GP_form_of_filtering_distribution} for a detailed derivation of \cref{eq:filtering_cov}.)

\subsection{Decomposition of the True Jacobian}    \label{subsec:JacobianEstimator}

Next, we give an explicit decomposition of the true Jacobian into the estimator $J$, the kernel prefactor $K$ and a sensitivity term $S$.
\begin{theorem}    \label{theorem:true_Jacobian_of_mtheta}
    Under \Cref{ass:f_linear_in_theta}, the true Jacobian $D \bs{m}_{\theta} \in \R^{M \times n}$ of $\theta \mapsto \bs{m}_{\theta}$ has the analytic form
    \begin{align}
        \label{eq:analyticJacobian}
        D \bs{m}_{\theta}
        &\defeq
        [ \nabla_{\theta} m(t_1), \dots, \nabla_{\theta} m(t_M) ]^{\intercal}
        =
        J + KS
        ,
    \end{align}
    where the \emph{sensitivity term} $S$ is defined by
    \begin{align}   \label{eq:def_S}
        S 
        \defeq
        \begin{bmatrix}
            \Lambda_1^{\intercal} \theta, \dots, \Lambda_{N}^{\intercal} \theta
        \end{bmatrix}^{\intercal}
        \
        \in \R^{N \times n}.
    \end{align}
    Here, $\Lambda_j = \begin{bmatrix}\lambda_{kl}(jh)\end{bmatrix}_{kl}$ is the $n \times n$ matrix with entries
    \begin{align}  \label{eq:def_lambda_kl}
        \lambda_{kl}(jh)
        \defeq
        \frac{\rd}{\rd x}f_l(m_{\theta}^-(jh))
        \cdot 
        \frac{\partial}{\partial \theta_k} m_{\theta}^-(jh).
    \end{align}
\end{theorem}
\begin{proof}
  See \Cref{proof:theorem:true_Jacobian_of_mtheta}.
\end{proof}
Thus, $KS$ is the exact approximation error of $J$.

\section{Bound on Approximation Error of $J$}   \label{sec:bound_on_appr_error}

In this section, we provide a bound on the approximation error of $J$ under the following assumptions.
\begin{assumption}  \label{ass:fi_regularity}
  The first-order partial derivatives of $f_i$, $1\leq i \leq N$, are bounded and globally $L$-Lipschitz, for $L>0$.
\end{assumption}
\Cref{ass:fi_regularity} is required to bound the global error of the ODE forward solution by \citet[Thm.~6.7]{KerstingSullivanHennig2018}.
\begin{assumption}\label{ass:barN}
  For the computation of $J$ we only use a maximum of $\bar{N} \leq N$ time points, for some finite $\bar{N} \in \N$. 
\end{assumption}
\Cref{ass:barN} precludes the condition number of the $K$ and $S$ from growing arbitrarily large, thereby preventing numerical instability.
While this restriction is necessary for \Cref{theorem:bound_on_approximation_error}, it is not relevant in practice because we are computing with a non-zero step size $h>0$ anyway so that many different parameters $\theta$ can be simulated.
\begin{theorem} \label{theorem:bound_on_approximation_error}
  If $\Theta \subset \R^n$ is compact and $R>0$, then it holds true, under Assumptions \ref{ass:f_linear_in_theta} to \ref{ass:barN}, that 
  \begin{align}   \label{eq:bound_on_approximation_error}
    \norm{J - D \bs{m}_{\theta}}
    \leq 
    C(T) \left ( \norm{\nabla_{\theta} x_{\theta}} + h \right )
  \end{align}  
  for sufficiently small $h>0$, where $C(T) > 0$ is a constant that depends on $T$.
\end{theorem}
\begin{proof}
  See \Cref{suppl:proof_bound_on_approximation error}.
\end{proof}
Intuitively, this upper bound can be thought of as a decomposition of the approximation error of the `sensitivity-unaware' estimator $J$ into a summand proportional to the ignored sensitivity $\norm{\nabla_{\theta} x_{\theta}}$ and the global integration error of the ODE filter, which is bounded by $C(T)h$ \citep[Thm.~6.7]{KerstingSullivanHennig2018}.

\section{Gradient and Hessian Estimators}  \label{sec:gradient_and_hessian_estimators}

We observe that the uncertainty-aware likelihood, \cref{subeq:uncertainty_aware_likelihood}, can be written in the form
\begin{align}  \label{eq:maximum_likelihood_objective}
    p(\bs{z}\given \theta)
    =
    \frac{e^{-E(\bs{z})}}{Z},
\end{align}
with evidence $Z > 0$ and negative log-likelihood
\begin{align}    \label{eq:def_log-likelihood}
    E(\bs{z})
    &\defeq 
    \frac{1}{2} 
    \left [\bs{z} - \bs{m}_{\theta} \right]^{\intercal}
    \left [ \bs{P} + \sigma^2 I_M \right ]^{-1}
    \left [ \bs{z} - \bs{m}_{\theta} \right ]
    \\
    &\stackrel{\cref{eq:mean_linear_in_param}}{=}
    \frac{1}{2} 
    \left [\bs{z} - x_0 \cdot \mathbbm{1}_M  - J \theta \right]^{\intercal}
    \left [ \bs{P} + \sigma^2 I_M \right ]^{-1}
    \notag
    \\
    &\phantom{\stackrel{\cref{eq:mean_linear_in_param}}{=}} \times
    \left [ \bs{z} - x_0 \cdot \mathbbm{1}_M  - J \theta \right ]    
    .
\end{align}
For a given value of the Jacobian estimator $J$, the thereby-implied gradient and Hessian estimators are, by application of the chain rule,
\begin{align}\label{eq:gradient_estimator_given_by_J}
  \hat\nabla_\theta  E(\bs{z})
  &\defeq
  -J^{\intercal} \left [ \bs{P} + \sigma^2 I_M \right ]^{-1} \left [ \bs{z} - \bs{m}_{\theta} \right ],
  \quad
  \text{and}
  \\
  \hat\nabla^2_\theta E(\bs{z})
  &\defeq
  J^{\intercal} \left [ \bs{P} + \sigma^2 I_M \right ]^{-1} J.
  \label{eq:Hessian_estimator_given_by_J}
\end{align}
(See \Cref{fig:gradient_and_Hessian_visualization} for a visualization of these estimators.)
\Cref{suppl:sec:Gradient_And_Hessian_Estimators_for_The_Bayesian_Case} provides versions of these estimators for Bayesian inference of $\theta$.
\begin{figure}[ht]
\centering
\input{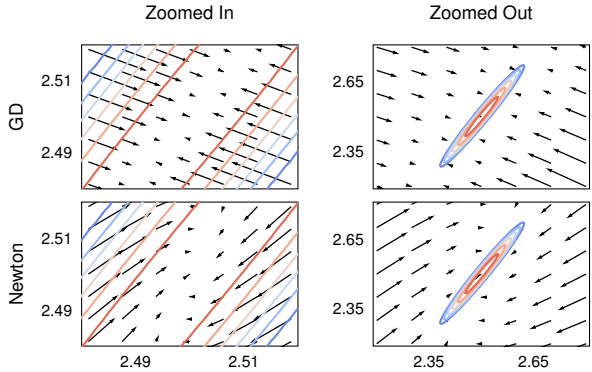} 
\caption{
Directions of gradient descent (GD) and Newton using \cref{eq:gradient_estimator_given_by_J,eq:Hessian_estimator_given_by_J}; around mode (left) and globally (right) of the likelihood, based on the logistic ODE. 
Globally, GD points more directly to the high-probability region.
Within this region, however, Newton is better directed to the mode.  
\label{fig:gradient_and_Hessian_visualization}}
\end{figure}

\section{New Gradient-Based Methods}   \label{sec:novel_gradientbased_algorithms}
By deriving gradient and Hessian estimators of the negative log-likelihood, we have removed the need for `likelihood-free' inference.
This enables the use of two classes of inference methods for $\theta$ which could not otherwise be applied: gradient-based \emph{optimization} and gradient-based \emph{sampling}.

\subsection{Gradient-Based Optimization}

In principle, all first and second-order optimization algorithms (e.g.~\citet{Bottou_Optimization_18}), are now applicable by \cref{eq:gradient_estimator_given_by_J,eq:Hessian_estimator_given_by_J}---such as (stochastic) gradient descent (GD), (stochastic) Newton (NWT), Gauss-Newton and natural Gradient descent.
This application of the estimators \eqref{eq:gradient_estimator_given_by_J} and \eqref{eq:Hessian_estimator_given_by_J} unlocks fast computation of single parameter estimates by maximum-likelihood estimation, as we demonstrate in the experiments (see \Cref{sec:experiments}).

\subsection{Gradient-Based Sampling}

Likewise, all gradient-based MCMC schemes are now available.
Classical gradient-based samplers include Langevin Monte Carlo (LMC) \citep{RobertsTweedie1996} and Hamiltonian Monte Carlo (HMC) \citep{betancourt2017conceptual}.
They are known to be more efficient than gradient-free samplers in finding and covering regions of high probability \citep[Section 30.1]{mackay2003information}.
While their standard form only makes use of gradients, more sophisticated versions include second-order information as well:
When the likelihood is ill-conditioned (i.e.~it varies much more quickly in some directions than others), it is advantageous to precondition the proposal distribution with a suitable matrix \citep{GirolamiCalderhead2011LangevinEqualHamiltonianMonteCarlo}.
A popular choice for the preconditioner is the Hessian \citep{QiMinka_HessianBased_2002}.
Hence, we can precondition LMC and HMC that use \cref{eq:gradient_estimator_given_by_J} as a gradient with the Hessian estimator from \cref{eq:Hessian_estimator_given_by_J}.   
For LMC, this leads to the proposal distribution
\begin{align}
\pi(\theta^{i+1} \given \theta^i) &= \theta^i - \rho [\hat\nabla^2_\theta E_{\theta^i}(\bs{z}))]^{-1}\hat\nabla_\theta E_{\theta^i}(\bs{z})) + \xi^i, \\
\xi^i &\sim \mathcal{N}(0, 2 \rho [\hat\nabla^2_\theta E_{\theta^i}(\bs{z}))]^{-1}),
\end{align}
where $\rho$ is the proposal width.
(Analogous formulas hold for HMC.)
Below, we refer to the so-preconditioned versions of LMC and HMC as PLMC and PHMC. 
In \Cref{sec:experiments}, we show that the gradient-based versions more aptly explore regions of high likelihood than their gradient-free counterparts.

\subsection{Algorithm}

The generic method that we propose is outlined in \Cref{alg:pseudocode}.
\begin{algorithm}
	\caption{Gradient-based sampling/optimization}
	\label{alg:pseudocode}
	\begin{algorithmic}[1]
 		\STATE Precompute $K$ and $(P + \sigma^2 I_M)^{-1}$  (see eqs.~\eqref{eq:def_K}, \eqref{subeq_b:likelihood})		
 		\STATE Initialize $\theta=\theta^0$
 		\REPEAT 
 			\STATE Solve ODE with $\theta$ (this generates $Y$; see eq.~\eqref{eq:def_data_factor}) \\
 			\STATE Compute $J = KY$ (see eq.~\eqref{eq:def_Jacobian})\\
 			\STATE Compute $[\hat\nabla_\theta E$, $\hat\nabla_\theta^2 E]$ (see eqs.~\eqref{eq:gradient_estimator_given_by_J}, \eqref{eq:Hessian_estimator_given_by_J})\\
 			\STATE Update $\theta$ with gradient-based sampler/optimizer\\
		\UNTIL{convergence/mixing}
	\end{algorithmic}
\end{algorithm} 
It includes all above-mentioned classical optimization and sampling methods (by a corresponding choice in Line 7).
The only difference, compared to all of these existing gradient-based methods, are the additional Lines 5 and 6 where we compute our gradient and Hessian estimators from \cref{eq:gradient_estimator_given_by_J,eq:Hessian_estimator_given_by_J}.

\subsection{Computational Cost}
The additional computational cost---on top of the employed classical optimization/sampling methods---is equal to the cost of computing the inserted gradient (and Hessian) estimators:
precomputation of $K$ (Line 1 in \Cref{alg:pseudocode}) 
requires the inversion of the $M$ kernel Gram matrices $\{\dKid,\ i=1,\dots,M\}$, which can have a maximum dimension of $(N-1) \times (N-1)$.
This inversion can, however, be executed in linear time since $\dkd$ is a Markov kernel \citep{hartikainen2010kalman}.
Hence, $K$ is in $\BO(MN)$ and, as $M \leq N$, in $\BO(N^2)$. The cost of inverting the $M \times M$ matrix $[ \bs{P} + \sigma^2 I_M  ]$ is in $\BO(N^3)$, as $M \leq N$.
Since $K$ and $\bs{P}$ are independent of $\theta$, this $\BO(N^3)$ cost is only required once.
The Jacobian estimator $J=KY$ 
(Line 5 in \Cref{alg:pseudocode}) 
is, by \cref{eq:def_Jacobian}, a matrix product of the precomputed kernel prefactor $K$ and the evaluation factor $Y$.
$Y$ is almost free, as it is by \cref{eq:def_data_factor} only composed of terms that the Gaussian ODE filter computes anyway; see \cref{eq:data_model} in \Cref{suppl:subsec:GaussianODEFiltering}.
Given $J$ and $[ \bs{P} + \sigma^2 I_M  ]^{-1}$, computing the gradient and Hessian estimators (Line 6 in \Cref{alg:pseudocode}) is of the same complexity as computing $J$.
Thus, the additional computational cost is in $\BO(N^3)$ w.r.t.~the number of time steps $N= T/h$ executed once and otherwise linear (but almost negligible) w.r.t.~the number of simulated parameters $\theta$.
As a large number of $\theta$ is usually required, the overall overhead is small.
%
%
%

\hansIR{Maybe write a paragraph on the kernel prefactor and that we could maybe use the following formula for efficient computation of the kernel prefactor $K$:
\begin{align*}
  {\dKTd }^{-1}
  =
  \frac{1}{h}
  \SmallbMatrix{ 
  2 & -1 & &  
  \\ 
  -1 & \smallddots & \smallddots & 
  \\ 
  & \smallddots & 2 & -1 
  \\
  &  & -1 & 1 
  }  
\end{align*}.
And then use matrix inversion lemma (Woodbury) to compute  $(\dKTd + R\cdot I_{N-1})^{-1}$.
}
\hansIR{Further speed-ups could be generated by \citet{GrigorievskiySarkka_SpInGP2019}.}

\subsection{Choice of Hyperparameters}   \label{subsec:choice_of_hyperparameters}

Recall that the parameters $\sigma$ and $R$ stem from the data and the accuracy of the ODE model \citep[Section 2.3]{KerstingSullivanHennig2018}, and that we only consider once-integrated Brownian motion priors in this paper.
Therefore, the only remaining hyperparameter is the diffusion scale $\sigma_{\text{dif}}$ which controls the width of the variance $\bs{P}$; see Supplements \ref{subsec:kernel_for_derivative_observations} and \ref{suppl:GP_form_of_filtering_distribution}.
There are two ways to set it: either as a local \citep[eq.~(46)]{schober2019} or as a global \citep[eq.~(41)]{TronarpKSH2019} maximum-likelihood estimate, which can both be computed from intermediate quantities of the forward solves. 

\section{Experiments}
\label{sec:experiments}
\davidD{There is a lot of mention of the tractability of computing the Hessian, but I'd like to see an asymptotic time cost.}

\chrisIR{Experimente: Wenn wir (zb in der Intro) die Key Contributions und deren Nutzen klar abgegrenzt haben, koennte man diese Punkte explizit in den Experimenten aufgreifen.}
\phIR{Ich stimme zu!}
To test the hypothesis that the gradient and Hessian estimators $[\hat\nabla_\theta  E(\bs{z}), \hat\nabla^2_\theta E(\bs{z})]$ of the log-likelihood are useful despite their approximate nature, we compare the new optimization and sampling methods from \Cref{sec:novel_gradientbased_algorithms}---which use these estimators as if exact---with the standard `likelihood-free' approach, i.e.~with random search (RS) optimization and random-walk Metropolis (RWM) sampling.

\subsection{Setup and Methods} \label{subsec:setup_and_methods}

As benchmark systems, we choose the popular Lotka--Volterra (LV) predator-prey model and the more challenging biochemical dynamics of glucose uptake in yeast (GUiY).
For more generality, we add the chemical protein signalling transduction (PST) dynamics which violate \Cref{ass:f_linear_in_theta} and have to be linearized.
We consider our hypothesis validated if the new gradient-based algorithms outperform the conventional `likelihood-free' methods (RS, RWM) on these three systems.
All datasets are, as in \cref{eq:p_z_given_x}, generated by adding Gaussian noise to the solution $x_{\theta^{\ast}}$ for some true parameter $\theta^{\ast}$.
\\
Out of the new family of gradient-based optimizers and samplers introduced in \Cref{sec:novel_gradientbased_algorithms}, we evaluate only the most basic ones: gradient descent (GD) and Newton's method (NWT) for optimization, as well as PLMC and PHMC for sampling.
This isolates the impact of the gradient and Hessian estimators more clearly.
The required gradient and Hessian estimators are computed as detailed above.
We employ the original fixed step-size RS by \citet{Rastrigin_RandomSearch_1963}, and the RWM version from \citet[Chapter 29]{mackay2003information}.
For all optimizers, we picked the best the step size and, for all samplers, the best proposal width within the interval $[10^{-16}, 10^{0}]$ which is wide enough to contain all plausible values.
To make these experiments an ablation study for the gradient and Hessian estimators, we use Gaussian ODE filtering as a forward solver in all methods---which is similar to classical solvers anyway \citep[Section 3]{schober2019}.
Since in all below experiments $\bs{P} \gg \sigma^2 I_M$, the gradient and Hessian estimates are scale-invariant w.r.t.~hyperparameter $\sigma_{\text{dif}}^2$, as can be seen from \cref{eq:gradient_estimator_given_by_J,eq:Hessian_estimator_given_by_J}:
In this regime, $\bs{P}$ simply scales the step-size of the gradient, and $\bs{P}$ cancels out of the Hessian, making it invariant to this scale.
The same applies in the regime $\bs{P} \ll \sigma^2 I_M$; adaptation of their relative scale, by choosing $\sigma_{\text{dif}}^2$ as in \Cref{subsec:choice_of_hyperparameters}, only matters when both error-sources are of comparable scale.

\subsection{Results}
\nico{Both metrics give equally good results. To make this clear, we will mention in Section 7.2 how close the final parameter estimates are to the true values and add plots about the error in the parameter estimates to the supplements.}
We evaluate the performance of these methods over the first few iterations (steps), comparing the values of the negative log-likelihood $E$ as well as the relative error in the parameter space, $\norm{\theta^i - \theta^{\ast}} / \norm{\theta^{\ast}}$.
For optimizers, low values in both metrics indicate success and, in fact, both are important: ODE inverse problems are inherently ill-posed and can have parameters with high likelihood and large inference error that fit the data as well as the true parameter. Finding these parameters would not be a failure of the algorithms, but a success, as they are a mode of the true posterior.

Samplers, on the other hand, try to identify and explore regions of high probability (the typical set); see e.g.~\citet[Section 2]{betancourt2017conceptual}.
We opt for plotting the relative error in the parameter space additionally to the negative log-likelihood values to emphasize that, once a sampler creates samples near the typical set, MCMC methods keep exploring suitable values instead of relying on a single estimate with high likelihood. 
Despite maintaining a low near-constant negative log-likelihood, the error in the parameter space of a sampler may have (some) variation.

The details and results for each benchmark systems are presented next, in ascending order of complexity.

\subsubsection{Lotka--Volterra}     \label{subsubsec:exp_lv}

First, we study the Lotka--Volterra (LV) ODE \citep{Lotka1978}
\begin{align}   \label{eq:lotka_volterra}
    \dot x_1 = \theta_1x_1 - \theta_2 x_1x_2, \quad \dot x_2  = -\theta_3 x_2 + \theta_4 x_1 x_2,
\end{align}
\martinIR{Could you reuse the same notation as in Eq. 1 here?}
the standard model for predator-prey dynamics.
\hansIR{It is maybe the most standard ODE benchmark problem.}
We used this ODE with initial value $x_0 = [20,20]$, time interval $[0,5]$ and true parameter $\theta^{\ast} = [1, 0.1, 0.1, 1]$. 
To generate data by \cref{eq:p_z_given_x}, we added Gaussian noise with variance $\sigma^2 = 0.01$ to the corresponding solution at time points $[0.5, 1, 1.5, 2, 2.5, 3., 3.5, 4., 4.5]$. 
The optimizers and samplers were initialized at $\theta^0 = [0.8, 0.2, 0.05, 1.1]$, and the forward solutions for all likelihood evaluations were computed with step size $h=0.05$.
In order to turn this $\theta^0$ into a useful initialization for the Markov chains, we accepted the first 45 states generated by PHMC and PLMC---the same would be counterproductive for RWM since a proposed sample may be further away from the region of nonzero probability. 
The results for optimization and sampling are depicted in \Cref{fig:LV}.
\begin{figure*}[t]
\centering
	\includegraphics[width=0.49\textwidth]{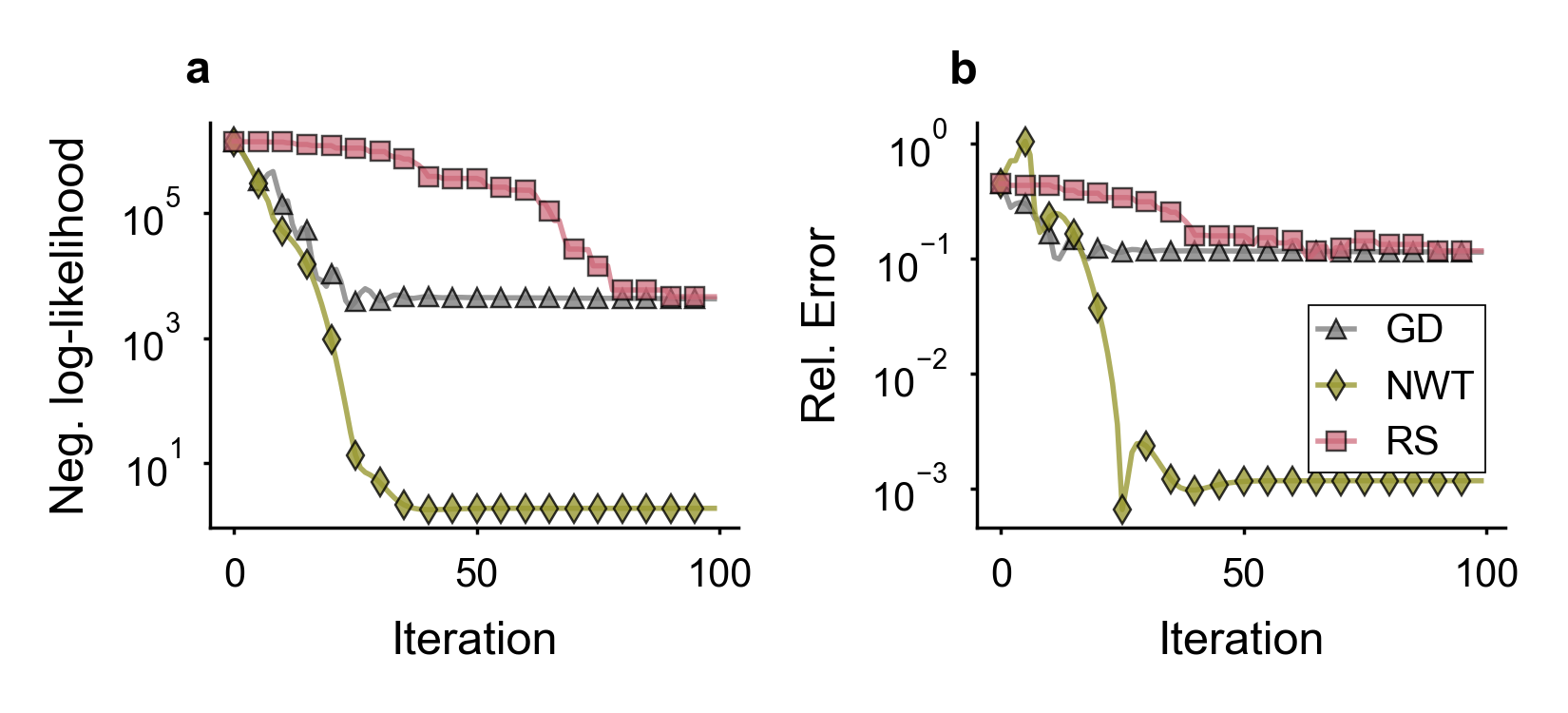}
	\includegraphics[width=0.49\textwidth]{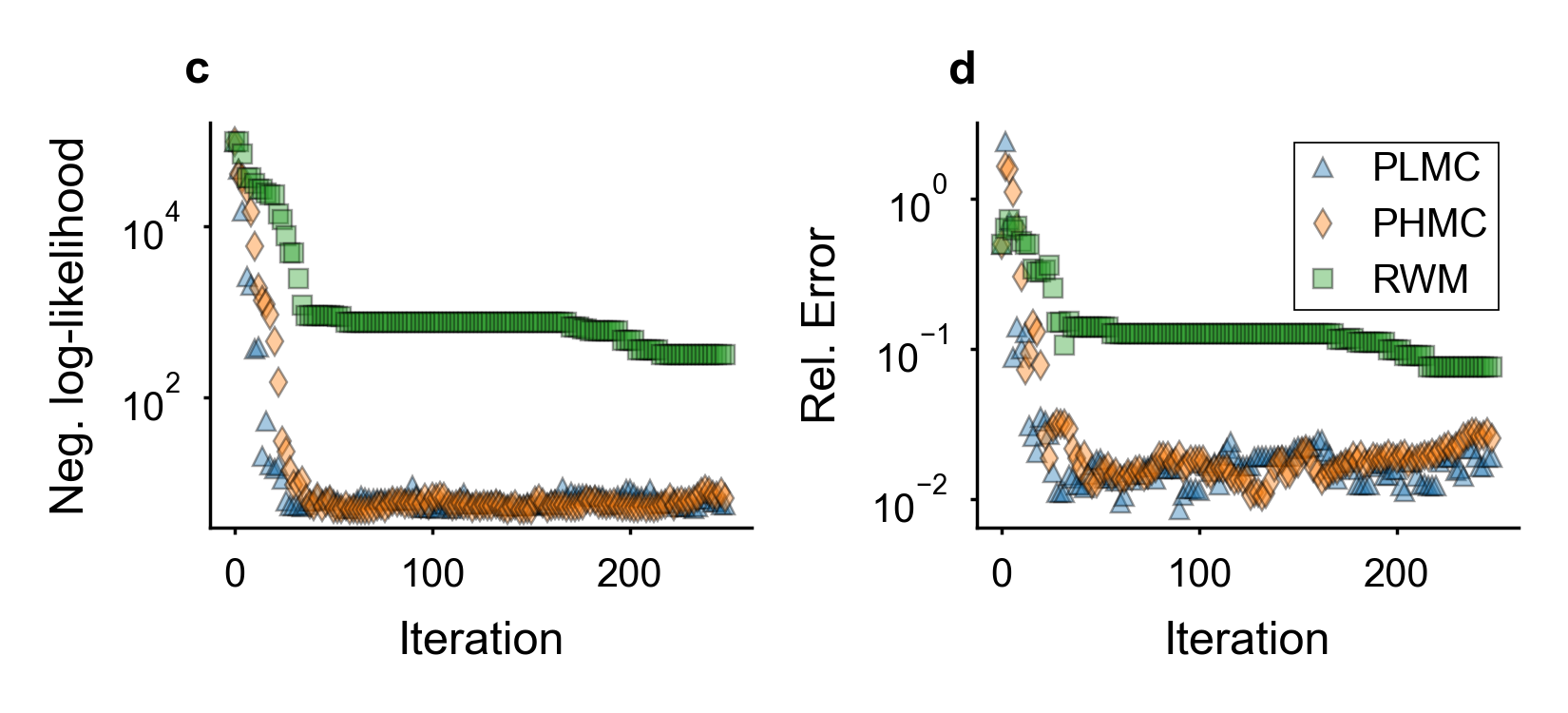}
\caption{
	 \textbf{Results for optimization (a, b) and sampling (c, d) on Lotka-Volterra}. 
Comparison of negative log-likelihood $E(\bs{z})=E_{\theta^i}(\bs{z})$ (a and c, resp.) and relative error $\norm{\theta^i - \theta^{\ast}} / \norm{\theta^{\ast}}$ (b and d, resp.). 100 iterations of optimization (only every fifth iteration has a marker) and 250 Metropolis-Hastings samples (only every other sample has a marker). 
\label{fig:LV}}
\end{figure*}
In the case of optimizers, NWT outperforms GD which, in turn, outperforms RS.
After roughly 25 samples, NWT generates iterations with relative error of less than $10^{-3}$.
While PLMC  and PHMC quickly reach and explore regions of high probability, RWM does not find likelihood values within the first 250 samples.
Thus, the gradient and Hessian estimators indeed appear to work well on LV.

\subsubsection{Protein Signalling Transduction}     \label{subsubsec:exp_pst}

Next, we consider the protein signalling transduction (PST) pathway.
It is governed by a combination of mass-action and Michaelis--Menten kinetics:
\begin{align*}
  \dot{S}
  &=
  -\theta_1 \times S - \theta_2 \times S\times R + \theta_3 \times RS,
  \\
  \dot{dS}
  &=
  \theta_1 \times S,
  \\
  \dot{R}
  &=
  - \theta_2 \times S \times R + \theta_3 \times RS + V \times \frac{Rpp}{K_m + Rpp},
  \\
  \dot{RS}
  &=
  \theta_2 \times S \times R - \theta_3 \times RS - \theta_4 \times RS,
  \\
  \dot{Rpp}
  &=
  \theta_4 \times RS - \theta_5 \times \frac{Rpp}{K_m + Rpp}.
\end{align*}
For more details, see \citet{VyshemirskyGirolami2008}.
Due to the ratio $\frac{Rpp}{K_m + Rpp}$, \Cref{ass:f_linear_in_theta} is violated.
As a remedy, we follow \citet{GorbachBauerBuhmann17} in defining the latent variables
$
  [x_1, x_2, x_3, x_4, x_5]
  \defeq
  [ S, dS, R, RS, \frac{Rpp}{K_m + Rpp}].
$
This gives rise to the new linearized ODE
\begin{align}   \label{eq:protein_signalling_transduction}
    \dot{x}_1 
    &= 
    -\theta_1 x_1 - \theta_2 x_1x_3 + \theta_3 x_4,
    \\
    \dot{x}_2 
    &= 
    \theta_1 x_1,
    \\
    \dot{x}_3 
    &= 
    -\theta_2 x_1 x_3 + \theta_3 x_4 + \theta_5  x_5,
    \\
    \dot{x}_4 
    &= 
    \theta_2 x_1 x_3 - \theta_3 x_4 - \theta_4 x_4,
    \\
    \dot{x}_5 
    &= 
    \theta_4 x_4 - \theta_5 x_5,
    \label{eq:last_equation}
\end{align}
which is an approximation of the original ODE, since \cref{eq:last_equation} ignores the factor $(K_m + R_{pp})^{-1}$.
We used this ODE with initial value $x_0 = [1,0,1,0,0]$ on time interval $[0,100]$.
We set the true parameter to $\theta^{\ast} = [0.07,  0.6,   0.05,  0.3,   0.017]$.
To generate the data by \cref{eq:p_z_given_x}, we added Gaussian noise with variance $\sigma^2 = 10^{-8}$ to the corresponding solution at time points $[1., 2., 4., 5., 7., 10., 15., 20., 30., 40., 50., 60., 80., 100.]$. 
The optimizers and samplers were initialized at $\theta^0 = [0.24 , 1.8  , 0.15 , 0.9  , 0.05]$, and the forward solutions for all likelihood evaluations were computed with step size $h=0.05$.
We use the same burn-in procedure as on the Lotka-Volterra example, accepting the first 100 samples.
\\
The results for optimization and sampling are depicted in \Cref{fig:PST}.
\begin{figure*}[t]
\centering
\includegraphics[width=0.49\textwidth]{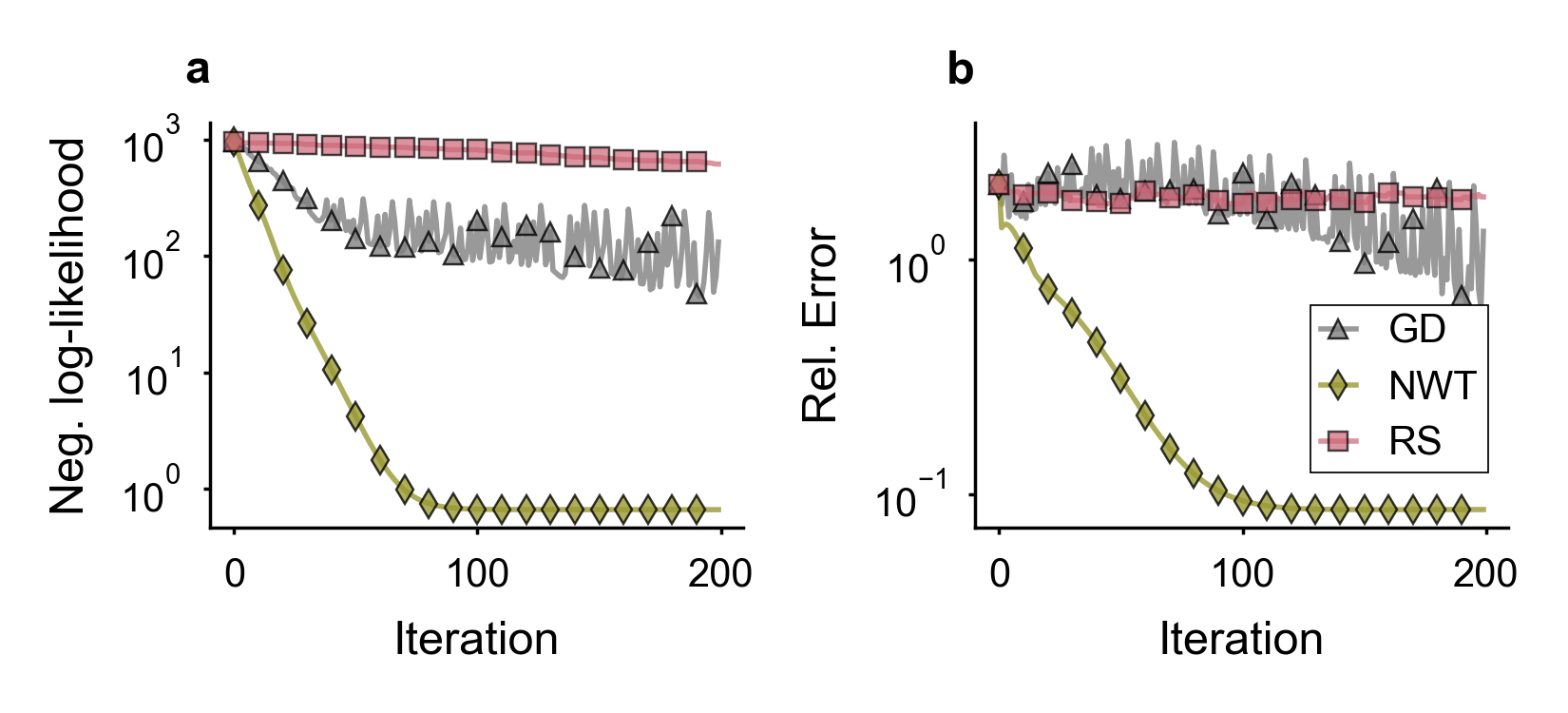}
\includegraphics[width=0.49\textwidth]{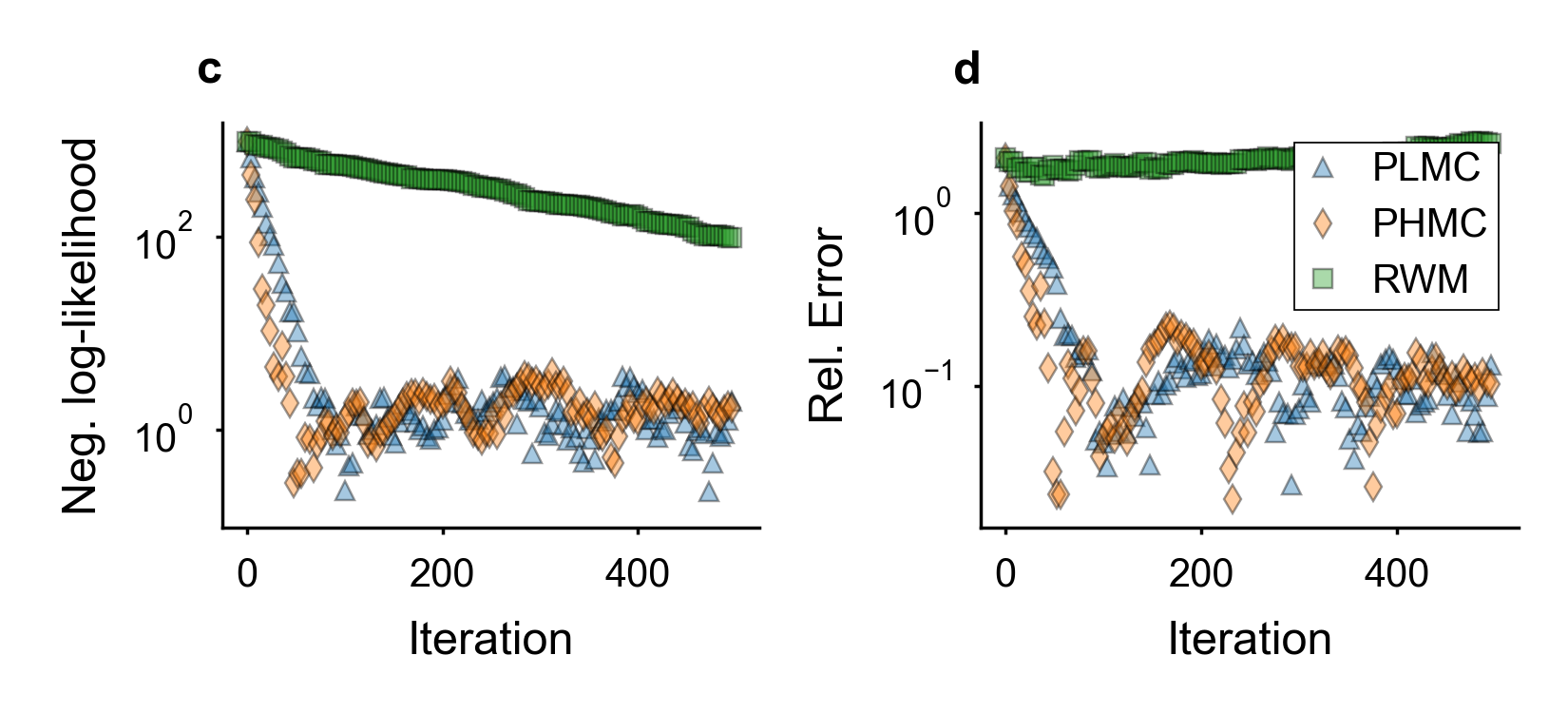}
\caption{
	\textbf{Results for optimization (a, b) and sampling (c, d) on PST}. 
	Comparison of negative log-likelihood $E(\bs{z})=E_{\theta^i}(\bs{z})$ (a and c, resp.) and relative error $\norm{\theta^i - \theta^{\ast}} / \norm{\theta^{\ast}}$ (b and d, resp.). 200 iterations of optimization (only every tenth iteration has a marker) and 500 Metropolis-Hastings samples (only every fourth sample has a marker). 
\label{fig:PST}}
\end{figure*}

Again, the new methods outperform the conventional ones in both optimization and sampling. 
For optimization, NWT converges particularly fast. The final estimate that is returned by NWT is, rounded to two digits, $\theta^{200} = (0.07, 0.60, 0.05,  0.30, 0.02)$, and hence recovers four out of five parameters exactly.
For sampling, both gradient-based samplers (after a fairly steep initial improvement) steadily stay in regions of high likelihood, while RWM only increases the likelihood in a much slower pace.
Hence, the gradient and Hessian estimators are beneficial on PST as well---although we had to linearize the ODE first.
\nico{Here, I think it would be interesting to say which "true" parameters the optimizers suggest. If my memory serves me right, all but $\theta_5$ are recovered correctly}

\subsubsection{Glucose Uptake in Yeast}       \label{subsubsec:exp_guiy}

Last, we examine the challenging biochemical dynamics of glucose uptake in yeast (GUiY), as seen in \citet{SchillingsSchwab_GUiY_2015}. 
This ODE is 9-dimensional, has 10 parameters, and satisfies \Cref{ass:f_linear_in_theta}; see \Cref{suppl:guiy} for a complete mathematical definition and parameter choices.
The results for optimization and sampling are depicted in \Cref{fig:guiy}.
\begin{figure*}[t]
\centering
\includegraphics[width=0.49\textwidth]{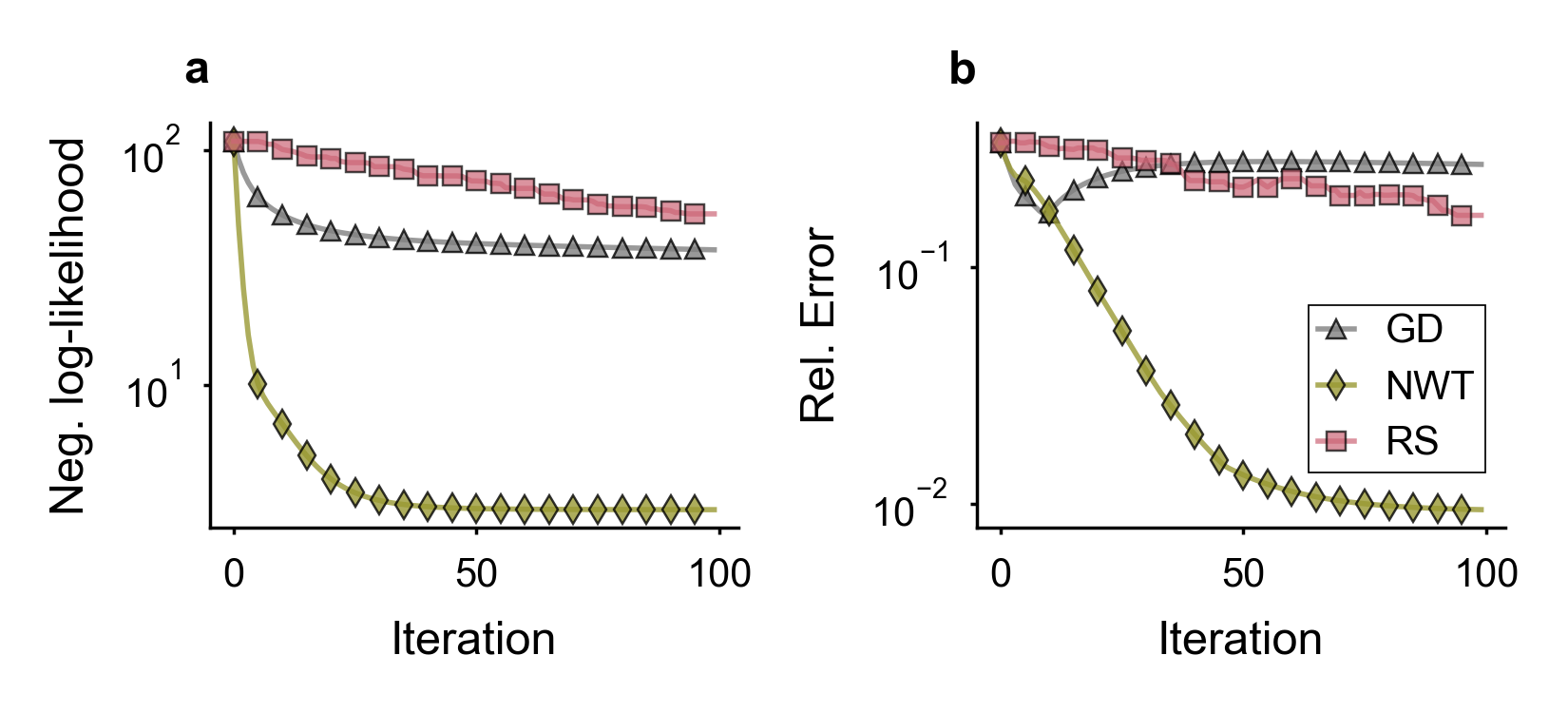}
\includegraphics[width=0.49\textwidth]{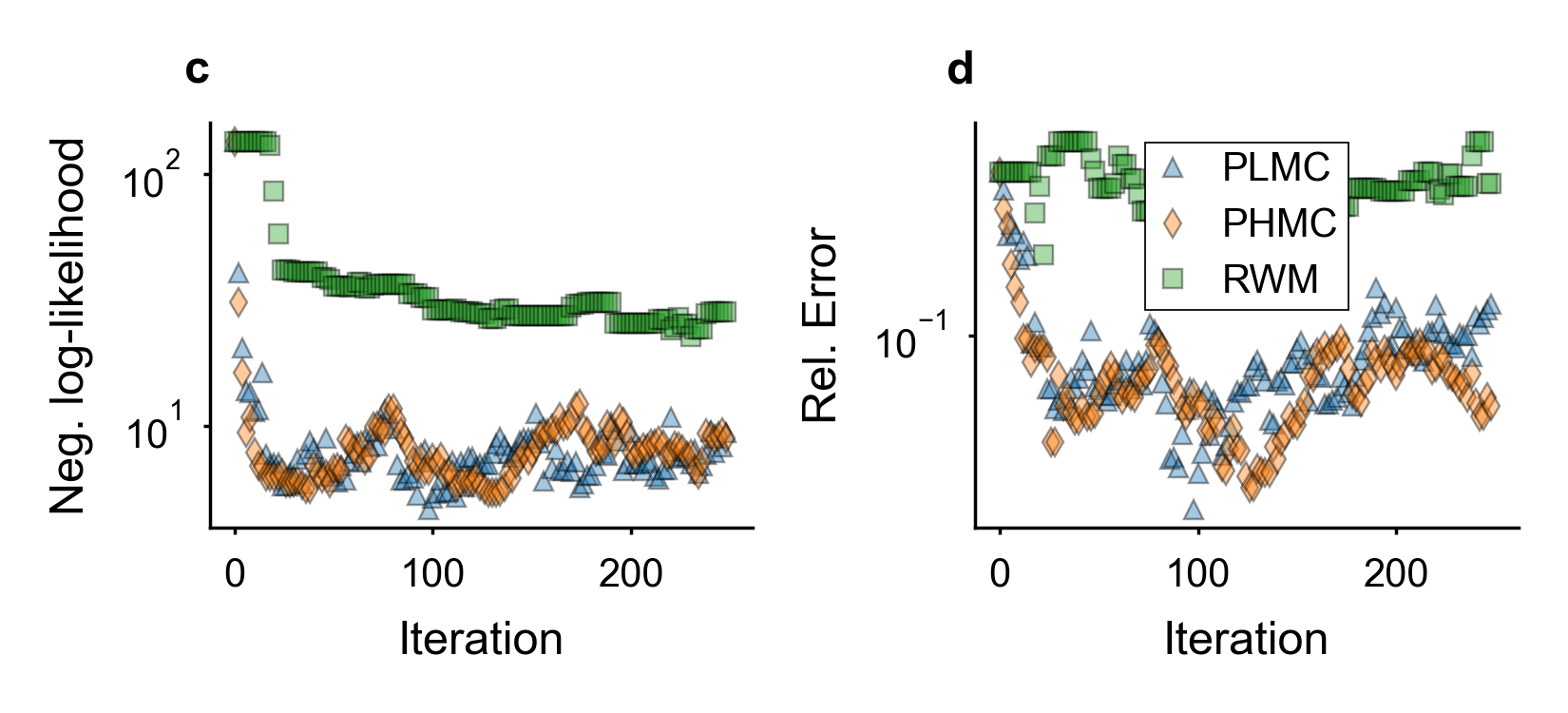}
\caption{
	\textbf{Results for optimization (a, b) and sampling (c, d) on GUiY}. 
	Comparison of negative log-likelihood $E(\bs{z})=E_{\theta^i}(\bs{z})$ (a and c, resp.) and relative error $\norm{\theta^i - \theta^{\ast}} / \norm{\theta^{\ast}}$ (b and d, resp.). 100 iterations of optimization (only every fifth iteration has a marker) and 250 Metropolis-Hastings samples (only every other sample has a marker). 
\label{fig:guiy}}
\end{figure*}

GD outperforms RS, and NWT converges even much faster than GD.
Remarkably, NWT already finds parameters that are exact up to two relative digits after only five iterations which would take RS extremely long on this 10 dimensional domain.
The gradient-based samplers (PLMC, PHMC), again, stay steadily within the region of significant likelihood, while RWM has difficulties sampling from this high dimensional problem in an efficient manner. 
Thus, this benchmark system also reaffirms the utility of the gradient and Hessian estimators.

\subsection{Summary of Experiments}

On all three benchmark ODEs, the Jacobian and Hessian estimator proved useful to speed up both sampling and optimization.
In the case of optimization, the new gradient-based methods consistently outperformed the classical random search.
Notably, the second-order optimization was always significantly more sample-efficient than plain gradient descent---which indicates that not only the gradient but also the Hessian estimator is accurate enough to be useful.
In the case of sampling, the gradient-based sampling methods, which were preconditioned by the Hessian, consistently outperformed the classical approach as well:
PLMC and PHMC steadily explored regions of elevated likelihood, while the conventional random-walk Metropolis methods hardly ever reached regions of nonzero probability  and wasted computational budget on less likely parameters.
\hansIR{Maybe add plot which shows why preconditioning is necessary and add a statement on the utility of the Hessian.}
Overall, we consider these experiments first evidence for the hypothesis that the proposed gradient-based methods require drastically fewer samples than the standard `likelihood-free' approach.
\nicoIR{It is evident that despite reaching positive likelihood values more quickly than RWM, roughly 10000 steps are needed whereas preconditioning reduces it down to roughly 12 steps. The problem settings (initial values, noise, etc.) are identical for all Lotka-Volterra experiments.}

\section{Related and Future Work}   \label{section:related_and_future_work}

The following research areas are particularly closely related to this paper.

{\bf Probabilistic numerical methods (PNMs)} \ There are two lines of work on PNMs for ODE forward problems: sampling- and filtering-based solvers; an up-to-date comparative discussion of these two approaches is given in \citet[Section 1.2.]{KerstingSullivanHennig2018}.
While this paper is the first to use filtering-based PNMs for inverse problems, there are previous methods---starting with \citet{o.13:_bayes_uncer_quant_differ_equat}---that use sampling-based solvers to integrate a non-Gaussian uncertainty-aware likelihood (cf.~the Gaussian \cref{subeq:uncertainty_aware_likelihood}) into a pseudo-marginal MCMC framework; see \citet{conrad_probability_2017}, \citet{Teymur2018a}, \citet{Lie17}, and \citet{AbdulleGaregnani17}. 
Notably, \citet{MatsudaMiyatake_ODEsIPsWithNumUQ_2019} recently proposed to model the numerical errors as random variables without explicitly employing PNMs.
On a related note, there are also first PNMs for PDE inverse problems; see \citet{Cockayne_PDEInverseProblems17} and \citet{OatesCockayne_PN_PDEs_2019}.

{\bf GP-surrogate methods} \ Modelling expensive likelihoods by GP regression is a common approach in statistics; see e.g.~\citet{Sacks_DesignAndAnalysisOfComputerExperiments_1989} and \citet{OHagan_GPSurrogateTutorial}.
Notably, \citet{meeds_welling_GPSABC} incorporated this approach into an ABC framework, and \citet{PerdikarisKarniadakis2016}, on the other hand, into a non-Bayesian setting by efficient global optimization.
While these methods also compute a GP approximation to the likelihood, they are fundamentally different as they globally model the likelihood with a GP (instead of constructing a local Gaussian approximation (see \cref{subeq:uncertainty_aware_likelihood}), and do not exploit the shape of the ODE at all.

{\bf Gradient Matching} \ This approach fits a joint GP model of the solution and its derivatives by conditioning on the ODE. 
Since introduced by \citet{calderhead2008accelerating}, it has received much attention in machine learning; see \citet{MacdonaldHusmeier_GradMatchingReview_2015} for a detailed review, \citet[Section 1]{WenkAbbiatiBauerOsborneKrause2019ODIN} for an up-to-date overview, and \citet{GorbachBauerBuhmann17} for a paper that uses a slightly stronger version of our \Cref{ass:f_linear_in_theta}.
As it avoids explicit numerical integration altogether, gradient matching is fundamentally different from our method (and PNMs in general).

{\bf Sensitivity analysis} \ This field studies the derivatives of ODE solutions with respect to parameters; see, e.g., \citet{Rackauckas2018SensitivityAnalysisInML} for an overview spanning continuous (adjoint) sensitivity analysis and automatic differentiation.
Therefore, the Jacobian estimator $J$ of the map $\theta \mapsto \bs{m}_{\theta} \approx \bs{x}_{\theta}$ from \cref{eq:def_Jacobian} can be interpreted as fast, approximate sensitivity analysis. 
This link is particularly interesting for modern machine learning, as sensitivity analysis is the mathematical corner stone of the recent advances by, e.g., \citet{ChenDuvenaud18} in training neural networks as ODEs.
It should be possible to use $J$ for neural ODEs---as well as for all other applications of sensitivity analysis.

\subsection*{Future Work}

We hope that this is the beginning of a new line of work on ODE inverse problems by ODE filtering.
Here, we only used Gaussian ODE filtering with once-integrated Brownian motion prior.
Future work could not only examine different priors \citep[Section 2.1]{KerstingSullivanHennig2018}, but also draw from the wide range of additional ODE filters (EKF, UKF, particle filter, etc.) that were unlocked by \citet{TronarpKSH2019}.
Notably, particle ODE filtering represents the belief over the ODE solution by a set of samples (particles), and could, therefore, be integrated in the above-mentioned existing framework for sampling-based PNMs. 
\\
The utility of the Jacobian estimator $J$ is, however, not limited to inverse problems.
As it constitutes fast, approximate sensitivity analysis, it should be compared with established methods, such as automatic differentiation and continuous sensitivity analysis \citep{Rackauckas2018SensitivityAnalysisInML}.
If $S$ (\cref{eq:def_S}) could also be estimated with low overhead, it is in light of \cref{eq:analyticJacobian} conceivable that the approximation error of $J$ could be further reduced.
\\
Either way, future work should examine which optimization and sampling methods are optimal---given that they received the (approximate) gradient and Hessian estimators $[\hat\nabla_\theta  E(\bs{z}), \hat\nabla^2_\theta E(\bs{z})]$.
For instance, the approximation error on these estimators might---according to \citet[Section 3.3]{Bottou_Optimization_18}---warrant optimization by stochastic methods such as SGD.
On a related note, it should be examined whether classical theorems on limit behavior of the employed optimization and MCMC methods remain true when using these estimators, and whether our approach is indeed applicable to ODEs that violate \Cref{ass:f_linear_in_theta}---as the results from \Cref{subsubsec:exp_pst} suggest.
Finally, this work should be, by the methods of lines \citep{Schiesser_ODEstoPDEsByMethodOfLines_2009}, extendable to PDEs and, by \citet{JohnSchober_GOODE_2019}, to boundary value problems.  

\hansIR{Future work should consider the Bayesian case, and compare with ABC literature such as \citet{meeds_welling_GPSABC}.}
\hansIR{compute kernel prefactor $K$ efficiently by [Grigorievskij and Sarkka]}

\section{Concluding Remarks}

We introduced a novel Jacobian estimator for ODE solutions w.r.t.~their parameters which implies approximate estimators of the gradient and Hessian of the log-likelihood.
Using these estimators, we proposed new first and second-order optimization and sampling methods for ODE inverse problems which outperformed standard `likelihood-free' approaches---namely random search optimization and random-walk Metropolis MCMC---in all conducted experiments.
Moreover, the employed Jacobian estimator constitutes a new method for fast, approximate sensitivity analysis.

\section*{Acknowledgements}

We thank the anonymous reviewers for their careful, constructive comments.
We thank Filip Tronarp and Katharina Ott for helpful feedback.

HK, NK and PH gratefully acknowledge financial support by the European Research Council through ERC StG Action 757275 / PANAMA; the DFG Cluster of Excellence “Machine Learning - New Perspectives for Science”, EXC 2064/1, project number 390727645; the German Federal Ministry of Education and Research (BMBF) through the T\"ubingen AI Center (FKZ: 01IS18039A); and funds from the Ministry of Science, Research and Arts of the State of Baden-W\"urttemberg. 
HK, NK and PH also gratefully acknowledge financial support by the German Federal Ministry of Education and Research (BMBF) through Project ADIMEM (FKZ 01IS18052B).


\bibliography{bibfile}
\bibliographystyle{icml2020}

\newpage

\appendix
\section{Short Introduction to Gaussian ODE Filtering} \label{suppl:introduction_ODE_Filtering}

\subsection{Gaussian Filtering for Generic Time Series} \label{suppl:introduction_Gaussian_Filtering}

In signal processing, a Bayesian Filter \citep[Chapter 4]{Sarkka2013} does Bayesian inference of the discrete state $\{x_i;\ i=1,\dots,N\} \subset \R^n$ from measurements $\{y_i;\ i=1,\dots,N\} \subset \R^n$ in a \emph{probabilistic state space model} consisting of 
\begin{align}
  \label{eq:dynamic_model}
  \text{a dynamic model} \quad x_i &\sim p(x_i \given x_{i-1}),\quad \text{ and}
  \\
  \text{a measurement model} \quad y_i &\sim p(y_i \given x_i).
  \label{eq:measurement_model}
\end{align}
Usually, the state $x_i$ is assumed to be the discretization of a continuous signal $x:[0,T] \to \R^n$ which is \emph{a priori} modeled by a stochastic process.
Absent very specific expert knowledge, this prior is usually chosen to be a linear time-invariant (LTI) stochastic differential equation (SDE):
\begin{align}   \label{eq:LTI-SDE}
  p(x) 
  \ \sim \
  X (t)
  =
  F 
  X (t)
  \ \rd t 
  + 
  L
  \ \rd B(t),
\end{align}
where $F$ and $L$ are the drift and diffusion matrix, respectively.
The corresponding dynamic model (\cref{eq:dynamic_model}) can be easily constructed by discretization of the LTI SDE (\cref{eq:LTI-SDE}), as described in \citet[Chapter 6.2]{SarkkaSolin2019}.
If an LTI SDE prior with Gaussian initial condition is used, $p(x)$ is a GP which implies a Gaussian dynamic model
\begin{align} \label{eq:GaussianDynamicModel}
  p(x_i \given x_{i-1})
  =
  \mathcal{N}(Ax_{i-1},Q)
\end{align}
for matrices $A,Q$ that are implied by $F,L$ from \cref{eq:LTI-SDE}.
If additionally the measurement model (\cref{eq:measurement_model}) is Gaussian, i.e.~
\begin{align}  \label{eq:linearGaussianMM}
  p(y_i \given x_i)
  =
  \mathcal{N}(Hx_i,R)
\end{align}
for matrices $H,R$, the filtering distributions $p(x_i \given y_{1:i} )$, $i=1,\dots,N$, can be computed by Gaussian filtering in linear time.
Note that the filtering distribution $p(x_i \given y_{1:i} )$ is not the full posterior distribution $p(x_i \given y_{1:N} )$ which can, however, also be computed in linear time by running a smoother after the filter. 
See e.g.~\citet{Sarkka2013} for more information.

\subsection{Gaussian ODE Filtering}   \label{suppl:subsec:GaussianODEFiltering}

A Gaussian ODE filter is simply a Gaussian filter, as defined in \Cref{suppl:introduction_Gaussian_Filtering}, with a specific kind of probabilistic state space model \cref{eq:dynamic_model,eq:measurement_model}, to infer the solution $x:[0,T] \to \R^d$ of the ODE \cref{eq:ODE}, at the discrete time grid $\{0 \cdot h, \dots, N \cdot h\}$ with step size $h>0$.
The dynamic model is---as usual, recall \cref{eq:LTI-SDE,eq:GaussianDynamicModel}---constructed from a GP defined by a LTI SDE that incorporates the available prior information on $x$.
The measurement model, however, is specific to ODEs as we will see next:
Recall that, after $i-1$ steps, the Gaussian filter has computed the $(i-1)$-th filtering distribution
\begin{align} \label{eq:filtering_distribution_suppl}
  p(x_{i-1} \given y_{1:i-1}) = \mathcal{N}(m_{i-1}, P_{i-1}),
\end{align}
which is Gaussian with mean $m_{i-1}$ and covariance matrix $P_{i-1}$, and computes the predictive distribution
\begin{align}
  p(x_{i} \given y_{1:i-1})
  =
  \mathcal{N}(m^-_{i}, P^-_{i})
\end{align}
by inserting \cref{eq:GaussianDynamicModel} into \cref{eq:filtering_distribution_suppl}.
Analogous to the logic
\begin{align}
  f(\hat{x}(t))
  \approx
  f(x(t))
  =
  \dot{x}(t) 
\end{align}
of classical solvers, the Gaussian ODE Filter treats evaluations at the predictive mean $m^-_{i}$---which is a numerical approximation like $\hat{x}$---as data on $\dot{x}(ih)$.
This yields the measurement model
\begin{align}
  p(y_i \given x_i)
  =
  \mathcal{N}( H x_i, R ),
\end{align}
with data 
\begin{align}  \label{eq:data_model}
  y_i 
  \defeq 
  f(m^-_i)
  \approx
  \dot{x}(ih).
\end{align}
The probabilistic state space model is thereby completely defined.
Gaussian ODE filtering is equivalent to running a Gaussian filter on this probabilistic state space model.
\\
For more details on Gaussian ODE filters, see \citet{KerstingSullivanHennig2018} or \citet{schober2019}.
An extension to more Bayesian filters---such as particle filters---is provided by \citet{TronarpKSH2019}.

\section{Equivalent Form of Filtering Distribution by GP Regression}  \label{suppl:EquivForm}

Recall from \Cref{suppl:introduction_ODE_Filtering} that any Gaussian filter computes a sequence of filtering distributions
\begin{align}   \label{eq:GP_regression_likelihood}
  p(x_i \given y_{1:i})
  =
  \mathcal{N}(m_i, P_i)
\end{align}
from a GP prior on $x$ \cref{eq:LTI-SDE} and a linear Gaussian measurement model (\cref{eq:linearGaussianMM}) with derivative data (\cref{eq:data_model}).
Hence, the classical framework for GP regression with derivative observations, as introduced in \citet{Solak_2003}, is applicable.
It \emph{a priori} models the state $x$ and its derivative $\dot{x}$ as a multi-task GP: 
\begin{align}   \label{eq:GP_regression_prior}
    p \left ( \begin{bmatrix} x \\ \dot{x} \end{bmatrix}  \right )
    =
    \mathcal{GP}\left(
    \begin{bmatrix} x \\ \dot{x} \end{bmatrix};
    \
    \begin{bmatrix} \mu \\ \dot{\mu} \end{bmatrix},
    \begin{bmatrix} k & \kd \\ \dk & \dkd \end{bmatrix}
    \right ),
\end{align}
with 
\begin{align}
    \dk = \frac{\de k(t,t^{\prime})}{\de t},\
    \kd = \frac{\de k(t,t^{\prime})}{\de t^{\prime}},\
    \dkd = \frac{\de^2 k(t,t^{\prime})}{\de t \de t^{\prime}}.
\end{align}

\subsection{Kernels for Derivative Observations}   \label{subsec:kernel_for_derivative_observations}

In this paper, we model the solution $x$ with a integrated Brownian motion kernel $k$ or, in other words, we model $\dot x$ by the Brownian Motion (a.k.a.~Wiener process) kernel, i.e.
\begin{align} \label{eq:defBMkernel}
  \dkd(t, t^{\prime})
  =
  \sigma_{\text{dif}}^2 \min(t,t^{\prime}),
  \qquad
  \forall t,t^{\prime} \in [0,T].
\end{align}
Here, $\sigma_{\text{dif}} > 0$ denotes the output variance which scales the diffusion matrix $L$ in the equivalent SDE (\cref{eq:LTI-SDE}). 
Integration with respect to both arguments yields the integrated Brownian motion (IBM) kernel
\begin{align}  \label{eq:defIBMkernel}
  k(t, t^{\prime})
  =
  \sigma_{\text{dif}}^2
  \left ( \frac{\min^3(t,t^{\prime})}{3} + \absval{t-t^{\prime}} \frac{\min^2(t,t^{\prime})}{2}  \right )
\end{align}
to model $x$.
The once-differentiated kernels in \cref{eq:GP_regression_prior} are given by
\begin{align}  \label{eq:defkdkernel}
  \kd(t, t^{\prime})
  = 
  \dk(t^{\prime}, t)
  &=
  \sigma_{\text{dif}}^2
  \begin{cases}
    t\leq t^{\prime}: \frac{t^2}{2},
    \\
    t > t^{\prime}: tt^{\prime} - \frac{{t^{\prime}}^2}2
  \end{cases}
  .
\end{align} 
A detailed derivation of \cref{eq:defBMkernel,eq:defIBMkernel,eq:defkdkernel} can be found in \citet[Supplement B]{schober2014nips}.

\subsection{GP Form of Filtering Distribution}  \label{suppl:GP_form_of_filtering_distribution}

Now, GP regression with prior (\cref{eq:GP_regression_prior}), likelihood (\cref{eq:GP_regression_likelihood}) and data $y_{1:i}$ yields an equivalent form of the filtering distribution \cref{eq:GP_regression_likelihood}:
\begin{align}
    m_i
    =
    &\mu + \kd(h:ih, ih)^{\intercal}
    \left [ \dKd(h:ih)  + R \cdot I_{i}  \right ]^{-1}
    \notag
    \\
    &\times
    \left [y_1 - \dot{\mu}(h) , \dots , y_i - \dot{\mu}(ih)  \right ]^{\intercal},
    \label{eq:deriv_m_linear_in_theta}
    \\
    P_i
    =
    &\SmallbMatrix{ k(h,h) & \dots & k(i h, i h) \\ \vdots & \ddots & \vdots \\ k(i h,h) & \dots & k(i h, i h) } 
    -
    \kd(h:ih, ih)^{\intercal}
    \notag
    \\
    &\times \left [ \dKd(h:ih)  + R \cdot I_l  \right ]^{-1}
    \kd(h:ih, ih),
    \label{eq:derivation_filtering_cov}
\end{align}
with $y_{1:i} = [ y_1, \dots, y_i ]^{\intercal}$, where we used the notations from \cref{eq:def_kd,eq:def_dKid}.
The derivation of \cref{eq:filtering_cov} is hence concluded by \cref{eq:derivation_filtering_cov}.
 
\subsection{Derivation of \Cref{eq:mean_linear_in_param}}  \label{suppl:derivation_of_linear_mean}
In this subsection, we will use the ODE-specific notation from above instead of the generic filtering notation---e.g.~$m_{\theta}(ih)$ instead of $m_i$, $f(m^-(ih))$ instead of $y_i$ etc.
To derive the missing \cref{eq:mean_linear_in_param}, we first observe that, by \cref{eq:deriv_m_linear_in_theta}, $m(ih)$ is linear in the data residuals:
\begin{align}
    m_{\theta}(ih)
    &=
    \mu + \beta_{ih}\ \times
    \label{eq:m_with_beta}
    \\
    &\phantom{\times}\left [f(m^-(h)) - \dot{\mu}(h) , \dots , f(m^-(ih)) - \dot{\mu}(ih)  \right ]^{\intercal}
    \notag
    \\
    \beta_{ih}
    &\defeq
    \kd(h:ih, ih)^{\intercal}
    \left [ \dKd(h:ih)  + R \cdot I_{i}  \right ]^{-1}.
    \notag
\end{align}
Now recall that, in ODE filtering, the prior mean in \cref{eq:GP_regression_prior} is set to be
$
    [ \mu, \dot{\mu} ]
    \equiv
    [ x_0; f(x_0) ]
$ 
(or $[ \mu, \dot{\mu} ] \equiv [m_0;f(m_0)]$ for some estimate $m_0$ of $x_0$, in the case of unknown $x_0$).
Consequently, application of \Cref{ass:f_linear_in_theta} to \cref{eq:m_with_beta} yields
\begin{align}
    m_{\theta}(ih)
    &=
    x_0 + J_{ih}
    \theta,
    \quad
    \text{with}
    \label{eq:mean_with_J_ih}
    \\
    J_{ih}
    &:=
    \beta_{ih}
    \SmallbMatrix{
    f_1(m_{\theta}^-(h)) - f_1(x_0) & \dots & f_n(m_{\theta}^-(h)) - f_n(x_0)
    \\
    \vdots & \ddots & \vdots
    \\
    f_1(m_{\theta}^-(ih)) - f_1(x_0) & \dots & f_n(m_{\theta}^-(ih)) - f_n(x_0)
    }
    \notag
    \\
    &=
    \beta_{ih}
    Y_{1:i}
    \quad
    ,
    \label{eq:def_J_ih}
\end{align}
where $Y_{1:i}$ denotes the first $i$ rows of $Y$; see \cref{eq:def_data_factor}.
We omit the dependence of $J_{ih}$ on $\theta$ to obtain a linear form.
Recall from \Cref{sec:likelihoods_by_GODEF} that we may w.l.o.g.~assume that the time points $\{t_1,\dots,t_M\}$ lie on the filter time grid, i.e.~$t_i = l_i h$ from some $l_i \in \N$.
Therefore, \cref{eq:mean_with_J_ih} implies
\begin{align}
    m_{\theta}(t_i)
    \stackrel{\cref{eq:def_tilde_kappa}}{=}
    x_0 + \tilde{\kappa}_i Y_{1:i}
    \stackrel{\cref{eq:def_kappa}}{=}
    x_0 + \kappa_i Y
    \label{eq:mean_linear_in_param_i}
\end{align}
for all data time points $t_i$, $i=1,\dots,M$.
Here, we used that $\tilde{\kappa}_i$ is equal to $\beta_{l_ih}$ by \cref{eq:def_tilde_kappa}.
We conclude the derivation of \cref{eq:mean_linear_in_param} by observing that the $i$-th entry of \cref{eq:mean_linear_in_param} reads \cref{eq:mean_linear_in_param_i} for all $i=1,\dots,M$.

\section{Proof of \Cref{theorem:true_Jacobian_of_mtheta}} \label{proof:theorem:true_Jacobian_of_mtheta}

\begin{proof}
    We start by computing the rows of 
    \begin{align}  \label{eq:Jacobian_m_proof}
        D\bs{m}_{\theta} = [ \nabla_{\theta} m(t_1), \dots, \nabla_{\theta} m(t_M) ]^{\intercal}.
    \end{align}
    By \cref{eq:mean_linear_in_param,eq:def_Jacobian} and the fact that the kernel prefactor $K$ does not depend on $\theta$, we obtain, for all $i=1,\dots,M$, that
    \begin{align}    
        \notag
        \nabla_{\theta} m(t_i)
        &=
        \nabla(\tilde{\kappa}(i)^{\intercal} v(\theta))
        \\
        &=
        \transJac{v(\theta)} \tilde{\kappa}(i) + \underbrace{\transJac{\tilde{\kappa}(i)}}_{=0} v(\theta)
        \\
        &=
        \transJac{v(\theta)} \tilde{\kappa}(i)
        ,
        \label{eq:grad_m_t}
    \end{align}
    with $v(\theta) = \tilde{Y} \theta$. Here, 
    \begin{align}
        \tilde{Y}
        =
        Y[1:l_i,:]
        =
        [Y_1(\theta), \dots, Y_{l_i}(\theta)]^{\intercal}
    \end{align} 
    is defined by 
    \begin{align} \label{eq:def_of_Y_j}
        Y_j(\theta) = [y_{j1}, \dots, y_{jn} ]^{\intercal}
        \ 
        \in \R^n,
    \end{align}
    the $j$-th row of $Y=Y(\theta)$ (recall \cref{eq:def_data_factor}), for $j=1,\dots,l_i$.
    Next, we again compute the rows of the missing Jacobian of \cref{eq:grad_m_t}
    \begin{align} \label{eq:D_v}
        D v(\theta) = [ \nabla_{\theta} [v(\theta)]_1, \dots, \nabla_{\theta} [v(\theta)]_{l_i} ]^{\intercal}
    \end{align}
    by the chain rule, for all $j \in \{1,\dots,l_i\}$: 
    \begin{align}   \label{eq:grad_v}
        \nabla_{\theta} [v(\theta)]_j
        =
        \nabla_{\theta} [ Y_j(\theta)^{\intercal} \theta ]
        =
        \transJac{Y_j(\theta)} \theta + Y_j(\theta).
    \end{align}
    Again, we compute the rows of the final missing Jacobian 
    \begin{align}   \label{eq:Jacobian_Y_j}
        D Y_j(\theta) = [\nabla_{\theta} y_{j1}(\theta), \dots, \nabla y_{jn}(\theta)]^{\intercal}.
    \end{align}
    The definition of $y_{ij}$ from \cref{eq:def_data_factor} implies, in the notation of \cref{eq:def_lambda_kl}, that
    \begin{align}
        \left[\nabla_{\theta} y_{jk}(\theta)\right]_l
        =
        \lambda_{lk}(jh)
        \label{eq:def_sensitivity_entry}
        , 
    \end{align}
    for all $l=1,\dots,n$.
    Now, we can insert backwards. First, we insert \cref{eq:def_sensitivity_entry} into \cref{eq:Jacobian_Y_j} which yields
    \begin{align}  \label{eq:Jacobain_Y_j_computed}
        D Y_j(\theta)
        =
        \Lambda_j
        ,
    \end{align}
    where $\Lambda_j = \begin{bmatrix}\lambda_{kl}(jh)\end{bmatrix}_{k,l=1,\dots,n}$.
    Second, insertion of \cref{eq:Jacobain_Y_j_computed} into \cref{eq:grad_v} provides that
    \begin{align}   \label{eq:gradient_v_theta_computed}
        \nabla_{\theta} [v(\theta)]_j
        =
        \Lambda_j^{\intercal} \theta + Y_j(\theta).
    \end{align} 
    Third, insertion of \cref{eq:gradient_v_theta_computed} into \cref{eq:D_v} implies that 
    \begin{align}  \label{eq:jacobian_v_computed}
        D v(\theta)
        =
        \left [ \Lambda_1^{\intercal} \theta, \dots, \Lambda_{l_i}^{\intercal} \theta \right ]^{\intercal}
        +
        Y[:l_i, :],
    \end{align}
    where
    \begin{align*}
        Y[:l_i,:]
        \stackrel{\cref{eq:gradient_v_theta_computed}}{=}
        [Y_1(\theta), \dots, Y_{l_i}(\theta)]^{\intercal}
        \stackrel{\cref{eq:def_of_Y_j}}{=}
        \SmallbMatrix
        {
        y_{11} & \dots & y_{1n}
        \\
        \vdots & \ddots & \vdots
        \\
        y_{l_i 1} & \dots & y_{l_i n} 
        }.
    \end{align*}
    Fourth, we insert \cref{eq:jacobian_v_computed} into \cref{eq:grad_m_t} and obtain
    \begin{align}   \label{eq:grad_m_t_i_computedI}
        \nabla_{\theta} m(t_i)
        &=
        \left (
        \left [ Y[:l_i, :] \right ]^{\intercal} +
        \begin{bmatrix}
            \Lambda_1^{\intercal} \theta, \dots, \Lambda_{l_i}^{\intercal} \theta
        \end{bmatrix}
        \right )
        \tilde{\kappa}_i
        \notag
        \\
        &=
        \left [ Y[:l_i, :] \right ]^{\intercal} \tilde{\kappa}_i
        +
        \begin{bmatrix}
            \Lambda_1^{\intercal} \theta, \dots, \Lambda_{l_i}^{\intercal} \theta
        \end{bmatrix}
        \tilde{\kappa}_i
        .
    \end{align}
    By \cref{eq:def_kappa}, it follows that
    \begin{align}
        \left [ Y[:l_i, :] \right ]^{\intercal} \tilde{\kappa}_i
        &\stackrel{\cref{eq:def_data_factor}}{=}
        Y^{\intercal} \kappa_i,
        \qquad \text{and}
        \\
        \begin{bmatrix}
            \Lambda_1^{\intercal} \theta, \dots, \Lambda_{l_i}^{\intercal} \theta
        \end{bmatrix}
        \tilde{\kappa_i}
        &\stackrel{\cref{eq:def_S}}{=}
        S^{\intercal} \kappa_i.
    \end{align}
    This implies via \cref{eq:grad_m_t_i_computedI} that
    \begin{align}  \label{eq:grad_m_t_i_computedII}
        \nabla_{\theta} m(t_i)
        =
        \left(Y^{\intercal}  + S^{\intercal} \right) \kappa_i,
    \end{align}
    Fifth and finally, we, by insertion of \cref{eq:grad_m_t_i_computedII} into \cref{eq:Jacobian_m_proof} and application of \cref{eq:def_K}, obtain
    \begin{align}
        D m_{\theta}
        =
        K(Y + S)
        \stackrel{\cref{eq:def_Jacobian}}{=}
        J + KS.
    \end{align}
\end{proof}

\section{Proof of \Cref{theorem:bound_on_approximation_error}}   \label{suppl:proof_bound_on_approximation error}

We first show some preliminary technical lemmas in \Cref{subsec:some_preliminary_lemmas} which are needed to prove bounds on $\norm{K}$ and $\norm{S}$ in \Cref{suppl:subsec:bound_K} and \Cref{suppl:subsec:bound_S}, respectively.
Having proved these bounds, the core proof of \Cref{theorem:bound_on_approximation_error} simply consists of combining them by \Cref{theorem:true_Jacobian_of_mtheta}, as executed in \Cref{subsec:proof_bound_on_approximation error}. 

\subsection{Preliminary lemmas}   \label{subsec:some_preliminary_lemmas}

The following lemma will be needed in \Cref{suppl:subsec:bound_K} to bound $\norm{K}$.
\begin{lemma} \label{lemma:inverse_of_sum_of_positive_definite_matrices}
  Let $Q > 0$ be a symmetric positive definite and $Q^{\prime} \geq 0$ a symmetric positive semi-definite matrix in $\R^{m\times n}$.
  Then, it holds true that
  \begin{align} \label{eq:inverse_of_sum_of_positive_definite_matrices}
    \norm{
    \left[ Q + Q^{\prime} \right ]^{-1}
    }_{\ast}
    \leq
    \norm{
    Q^{-1}
    }_{\ast},
  \end{align}
  for the nuclear norm
  \begin{align}
    \norm{
    A
    }_{\ast}
    =
    \Tr{\sqrt{A^{\ast} A } }
    =
    \sum_{i=1}^{m \wedge n} \sigma_i(A)
    ,
  \end{align}
  where $\sigma_i(A)$, $i\in \{1,\dots,m \wedge n\}$, are the singular values of $A$.
\end{lemma}
\hansIR{The below proof is a modified version of question 2018849 on math stackexchange.}
\begin{proof}
  Recall that, for all symmetric positive semi-definite matrices, the singular values are the eigenvalues.
  Therefore
  \begin{align}
    \norm{
    \left[ Q + Q^{\prime} \right ]^{-1}
    }_{\ast}
    &=
    \sum_{i=1}^{m \wedge n} \frac{1}{\lambda_i(Q + Q^{\prime})}
    \notag
    \\
    &\leq
    \sum_{i=1}^{m \wedge n} \frac{1}{\lambda_i(Q)}
    =
    \norm
    {
    Q^{-1}
    }_{\ast}.
    \label{eq:nuclear_norm_bound}
  \end{align}
  In \cref{eq:nuclear_norm_bound}, we exploited the fact that $Q \leq Q + Q^\prime$ (i.e.~that $(Q + Q^\prime) -  Q = Q^{\prime}$ is positive semi-definite) and therefore $\lambda_i(Q) \leq \lambda_i(Q + Q^\prime)$ for ordered eigenvalues $\lambda_1(Q) \leq \dots \leq \lambda_{m \wedge n}(Q)$ counted by algebraic multiplicity. This fact is an immediate consequence of Theorem 8.1.5. in \citet{golub1996matrix}.
\end{proof}
The next lemma will be necessary to prove a bound on $\norm{S}$ in \Cref{suppl:subsec:bound_S}.
\begin{lemma} \label{lemma:regularity_carries_over_to_param_derivative}
  Let $g(x, \lambda) \in C \left ( [0,T] \times \Lambda; \R \right )$ on non-empty compact $\Lambda \subset \R^n$ with continuous first-oder partial derivatives w.r.t.~the components of $\lambda$.
  If
  \begin{align}
    \sup_{\lambda \in \Lambda}
    g(x,\lambda)
    \in
    \mathcal{O}(h(x))
  \end{align}
  for some constant $C > 0$ and some strictly positive $h:[0,T] \to \R$, then also
  \begin{align}
    \sup_{\lambda \in \interior{\Lambda}}
    \absval{\frac{\partial}{\partial \lambda_k} g(x,\lambda)}
    \
    \in 
    \mathcal{O}(h(x)),
  \end{align}
  where $\interior{\Lambda}$ denotes the interior of $\Lambda$.
\end{lemma}
\begin{proof}
  Assume not.
  Then, there is a $k \in \{1,\dots,n\}$ and a $\tilde{\lambda} \in \interior{\Lambda}$ such that
  \begin{align} 
    \absval{\frac{\partial}{\partial \lambda_k} g(x,\tilde{\lambda})}
    \
    \notin
    \mathcal{O}(h(x)).
  \end{align}
  Since, for all $x \in [0,T]$, $\frac{\partial}{\partial \lambda_k}(x,\cdot)$ is uniformly continuous over the bounded domain $\interior{\Lambda}$, there is a $\delta > 0$ such that
  \begin{align}  \label{eq:proof:assume_not_BO}
    \absval{\frac{\partial}{\partial \lambda_k} g(x,\tilde{\lambda})}
    \
    \notin
    \mathcal{O}(h(x)),
    \quad
    \text{for all}\ \lambda \in B_{2\delta}(\tilde{\lambda}).
  \end{align}
  \hansIR{I am not super sure about the previous step, but I have given up to prove it more rigorously. But I do believe that \Cref{eq:proof:assume_not_BO} holds true. Nico agrees with this assessment. If somebody knows a proof, let me know =) !}
  Let us w.l.o.g.~(otherwise consider $-g$) assume that
  \begin{align}  \label{eq:proof_lemma_wlog}
    \frac{\partial}{\partial \lambda_k} g(x,\tilde{\lambda})
    \geq
    0,
    \quad
    \text{for all}\ \lambda \in B_{2\delta}(\tilde{\lambda}).
  \end{align}
  Now, on the one hand, we know by the fundamental theorem of calculus that
  \begin{align}
    &\int_{-\delta}^{0} \frac{\partial }{\partial \lambda_k} g(x_n, \tilde{\lambda} + \tilde{\delta} e_k) \ \rd \tilde{\delta} 
    \notag
    \\
    &\qquad
    =
    \underbrace{g(x, \tilde{\lambda})}_{\in \mathcal{O}(h(x))}
    -
    \underbrace{g(x, \tilde{\lambda} - \delta e_k)}_{\in \mathcal{O}(h(x))}
    \
    \in
    \mathcal{O}(h(x)).
    \label{eq:contradictionII}
  \end{align}
  However, on the other hand, we know from our assumption that
  \begin{align}
    0
    &\stackrel{\cref{eq:proof_lemma_wlog}}{\leq}
    \int_{-\delta}^{0} \frac{\partial }{\partial \lambda_k} g(x_n, \tilde{\lambda} + \tilde{\delta} e_k) \ \rd \tilde{\delta}
    \\
    &\leq
    \int_{-\delta}^{0} \underbrace{\absval{\frac{\partial }{\partial \lambda_k} g(x_n, \tilde{\lambda} + \tilde{\delta} e_k)}}_{\notin \mathcal{O}(h(x)),\ \text{by}\ \cref{eq:proof:assume_not_BO}} \ \rd \tilde{\delta}
    \
    \notin 
    \mathcal{O}(h(x))
    ,
  \end{align}
  which implies
  \begin{align}  \label{eq:contradictionI}
    &\int_{-\delta}^{0} \frac{\partial }{\partial \lambda_k} g(x_n, \tilde{\lambda} + \tilde{\delta} e_k) \ \rd \tilde{\delta} 
    \
    \notin
    \mathcal{O}(h(x)).
  \end{align}
  The desired contradiction is now found between \cref{eq:contradictionI,eq:contradictionII}.
\end{proof}

\subsection{Bound on $\norm{K}$} \label{suppl:subsec:bound_K}

\begin{lemma} \label{lemma:boundK}
  Under \Cref{ass:barN} and for all $R>0$, it holds true that
  \begin{align}
    \label{eq:boundK}
    \norm{K}
    \leq
    C(T),
  \end{align}
  where $C(T) > 0$ is a constant that depends on $T$.
\end{lemma}
\begin{proof}
  First, recall \cref{eq:def_K,eq:def_kappa,eq:def_tilde_kappa,eq:def_kd,eq:def_dKid} and observe that
  \begin{align*}
    \norm{\kd(h:t_i, t_i)}
    \leq
    C \frac{\sigma^2}{2} \norm{\begin{bmatrix} h^2, \dots, T^2 \end{bmatrix}}_{\infty}
    = 
    C\left ( 2^{-\frac 12} \sigma T \right )^2
    ,
  \end{align*}
  for all $i = 1,\dots,M$.
  Second, \Cref{lemma:inverse_of_sum_of_positive_definite_matrices} implies that
  \begin{align*}
    \norm{
    \left [ \dKid  + R \cdot I_{l_i}  \right ]^{-1}
    }
    \stackrel{\cref{eq:inverse_of_sum_of_positive_definite_matrices}}{\leq} 
    C \norm{ R^{-1} \cdot I_{l_i-1} }_{\ast}
    \\
    \vspace{4cm}\leq
    C\norm{ R^{-1} \cdot I_{\bar{N} - 1} }_{\ast}
    \leq C R \bar{N}.
  \end{align*}
  Now, by \cref{eq:def_kappa}, we observe
  \begin{align}
    \norm{
    \kappa_i
    }_1
    &=
    \norm{
    \tilde{\kappa}_i
    }_1
    \notag
    \\
    &\leq
    \norm{
    \left [ \dKid  + R \cdot I_{l_i}  \right ]^{-1}
    }
    \cdot
    \norm{
    \kd(h:t_i,t_i)
    }
    \notag
    \\
    &\leq
    C(T)
    ,
    \label{eq:bound_kappa_i}
  \end{align}
  where we inserted the above inequalities in the last step.
  Finally, we obtain \cref{eq:boundK} by plugging \cref{eq:bound_kappa_i} into
  \begin{align}
    \norm{K}
    \leq
    C \norm{K}_{\infty}
    \stackrel{\cref{eq:def_K}}{=}
    \max_{1 \leq i \leq M} \norm{\kappa_i}_1.
  \end{align}
\end{proof}

\subsection{Bound on $\norm{S}$} \label{suppl:subsec:bound_S}

Before estimating $\norm{S}$, we need to bound how far the entries of $S$ (recall \cref{eq:def_S}) deviate from the true sensitivities $\frac{\partial}{\partial \theta_k} x_{\theta}(T)$.

\begin{lemma} \label{lemma:sup_bound_on_sensitivity_difference}
  If $\Theta \subset \R^n$ is compact, then it holds true, under Assumptions \ref{ass:f_linear_in_theta} and \ref{ass:fi_regularity}, that
  \begin{align}  \label{eq:sup_bound_on_sensitivity_difference}
    \sup_{\theta \in \interior{\Theta}} 
    \norm{ \frac{\partial}{\partial \theta_k} m_{\theta}^-(T)   - \frac{\partial}{\partial \theta_k} x_{\theta}(T)   }
    \
    \in
    \BO(h).
  \end{align}      
\end{lemma}
\begin{proof}
  First, recall that the convergence rates of $\mathcal{O}(h)$ provided by Theorem 6.7 in \citet{KerstingSullivanHennig2018} only depend on $f$ through the dependence of the constant $K(T) > 0$ on the Lipschitz constant $L$ of $f$.
  But this $L$ is independent of $\theta$ by \Cref{ass:f_linear_in_theta}.
  Hence, Theorem 6.7 from \citet{KerstingSullivanHennig2018} yields under \Cref{ass:fi_regularity} that
  \begin{align}  \label{eq:prerequisite_carryoverlemma}
    \sup_{\theta \in \interior{\Theta}} 
    m_{\theta}^-(T)   -  x_{\theta}(T)
    \
    \in
    \BO(h).
  \end{align}
  Moreover, Theorem 8.49 in \citet{KelleyPeterson2010TheoryOfODEs} is applicable under \Cref{ass:f_linear_in_theta} and implies that $x_{\theta}(t)$ is continuous and has continuous first-order partial derivatives with respect to of $\theta_k$.
  By construction---recall \cref{eq:mean_linear_in_param}---the filtering mean $m_{\theta}(t)$ has the same regularity too.
  Hence, application of \Cref{lemma:regularity_carries_over_to_param_derivative} with
  $x = h$, $\Lambda = \Theta$, $\lambda = \theta$, $g(x,\lambda) = m_{\theta}^-(T)   -  x_{\theta}(T)$ is possible,
  which yields \cref{eq:sup_bound_on_sensitivity_difference} from \cref{eq:prerequisite_carryoverlemma}.
\end{proof}
\begin{lemma} \label{lemma:boundS}
  If $\Theta \subset \R^n$ is compact, then it holds true, under Assumptions \ref{ass:f_linear_in_theta} to \ref{ass:barN}, that
    \begin{align}
    \norm{S}
    \leq
    C\left ( \norm{\nabla_{\theta} x_{\theta}} + h \right ),
  \end{align}
  for sufficiently small $h>0$.
\end{lemma}
\begin{proof}
  By \Cref{ass:barN} and the equivalence of all matrix norms, we observe
  \begin{align}
    \norm{S} 
    &\leq
    C \norm{S}_2
    =
    C \norm{S^{\intercal}}_2
    \leq
    C \norm{S^{\intercal}}_{2,1}
    \\
    &\stackrel{\cref{eq:def_S}}{=}
    C \sum_{j=1}^{\bar N} \norm{\Lambda_j^{\intercal} \theta}_2
    \\
    &\leq
    C \sum_{j=1}^{\bar N} \norm{\Lambda_j^{\intercal}}_2
    \underbrace{\norm{\theta}_2}_{\leq C, \text { since $\Theta$ bounded}}
    ,
  \end{align}
  where $\norm{\cdot}_{2,1}$ denotes the $L_{2,1}$ norm.
  We conclude, using \Cref{ass:fi_regularity} and \Cref{lemma:sup_bound_on_sensitivity_difference}, that
  \begin{align}
    \norm{\Lambda_j^{\intercal}}_2
    &\stackrel{\cref{eq:def_lambda_kl}}{\leq} 
    L \max_{jk}\left[ \frac{\partial}{\partial \theta_k} m_{\theta}^-(jh) \right ]
    \\
    &\stackrel{\cref{eq:sup_bound_on_sensitivity_difference}}{\leq}
    C\left ( \norm{\nabla_{\theta} x_{\theta}} + h \right ).
  \end{align}
\end{proof}

\subsection{Proof of \Cref{theorem:bound_on_approximation_error}}  \label{subsec:proof_bound_on_approximation error}
\begin{proof}
  By \Cref{theorem:true_Jacobian_of_mtheta} and the sub-multiplicativity of the induced $p$-norm $\norm{\cdot}_p$, we observe that
    \begin{align}
    \norm{J - D \bs{m}_{\theta}}
    &=
    \norm{KS}  
    \leq
    C
    \norm{KS}_{p}
    \leq
    \norm{K}_{p}
    \norm{S}_{q}
    \notag
    \\
    &\leq
    C \norm{K} \norm{S}
    ,
  \end{align}
  for some $p,q \geq 1$.
  Application of \Cref{lemma:boundK,lemma:boundS} concludes the proof.
\end{proof}

\section{Gradient and Hessian Estimators for the Bayesian Case} \label{suppl:sec:Gradient_And_Hessian_Estimators_for_The_Bayesian_Case}

In the main paper, we only consider the maximum likelihood objective; see \cref{eq:maximum_likelihood_objective}.
Nonetheless, the extension to the Bayesian objective, with a prior $\pi(\theta)$, is straightforward:
\begin{align*}
  - \log \left ( p(\bs{z} \given \theta) \pi(\theta) \right )
  &=
  - \log \left ( p(\bs{z} \given \theta) \right )
  -
  \log \left ( \pi(\theta) \right )
\end{align*} 
Accordingly, the gradients and Hessian of this objective are
\begin{align*}
  \nabla_\theta \left [ - \log \left ( p(\bs{z} \given \theta) \pi(\theta) \right ) \right ]
  &\stackrel{\cref{eq:gradient_estimator_given_by_J}}{=}
  \hat\nabla_\theta  E(\bs{z}) - \nabla_{\theta} \log \left ( \pi(\theta) \right ),
  \\
  \nabla^2_\theta \left [ - \log \left ( p(\bs{z} \given \theta) \pi(\theta) \right ) \right ]
  &\stackrel{\cref{eq:Hessian_estimator_given_by_J}}{=}
  \hat\nabla^2_\theta E(\bs{z}) - \nabla^2_{\theta} \log \left ( \pi(\theta) \right ).
\end{align*}
Hence, for a Gaussian prior $\pi(\theta) = \mathcal{N}(\theta; \mu_\theta, V_\theta)$, the Bayesian version of the gradients and Hessian estimators in \cref{eq:gradient_estimator_given_by_J,eq:Hessian_estimator_given_by_J} are hence given by
\begin{align}
  \hat\nabla_\theta  E(\bs{z})_{\text{Bayes}}
  &\defeq
  -J^{\intercal} \left [ \bs{P} + \sigma^2 I_M \right ]^{-1} \left [ \bs{z} - \bs{m}_{\theta} \right ]
  \notag
  \\
  &\phantom{=}
  - V_{\theta}^{-1} \left [ \theta - \mu_\theta   \right ]
  ,
  \quad \text{and}
  \\
  \hat\nabla^2_\theta E(\bs{z})_{\text{Bayes}}
  &\defeq
  J^{\intercal} \left [ \bs{P} + \sigma^2 I_M \right ]^{-1} J
  + 
  V_{\theta}^{-1}.
\end{align}

\section{Glucose Uptake in Yeast}  \label{suppl:guiy}

The Glucose uptake in yeast (GUiY) is described by mass-action kinetics.
In the notation of \citet{SchillingsSchwab_GUiY_2015}, the underlying ODE is given by:
\begin{align*}
  \dot x_{\text{Glc}}^e
  &=
  -k_1 x_E^e x_{\text{Glc}}^e + k_{-1} x_{\text{E--Glc}}^e
  \\
  \dot x_{\text{Glc}}^i 
  &=
  -k_2 x_E^i x_{\text{Glc}}^i  + k_{-2} x_{\text{E--Glc}}^i
  \\
  \dot x_{\text{E--G6P}}^i 
  &=
  k_4 x_E^i x_{\text{G6P}}^i  +  k_{-4} x_{\text{E--G6P}}^i
  \\
  \dot x_{\text{E--Glc--G6P}}^i
  &=
  k_3 x_{\text{E--Glc}}^i x_{\text{G6P}}^i  -  k_{-3} x_{\text{E--Glc--G6P}}^i
  \\
  \dot x_{\text{G6P}}^i
  &=
  - k_3 x_{\text{E--Glc}}^i x_{\text{G6P}}^i  +  k_{-3} x_{\text{E--Glc--G6P}}^i
  \\
  &\phantom{=} - k_4 x_E^i x_{\text{G6P}}^i + k_{-4} x_{\text{E--Glc}}^i
  \\
  \dot x_{\text{E--Glc}}^e
  &=
  \alpha \left ( x_{\text{E--Glc}}^i - \dot x_{\text{E--Glc}}^e \right )
  +
  k_1 x_E^e x_{\text{Glc}}^e  
  \\
  &\phantom{=} -  k_{-1} x_{\text{E--Glc}}^e
  \\
  \dot x_{\text{E--Glc}}^i
  &=
  \alpha \left ( x_{\text{E--Glc}}^e - \dot x_{\text{E--Glc}}^i \right )
  - k_3 x_{\text{E--Glc}}^i x_{\text{G6P}}^i
  \\
  &\phantom{=}
  + k_{-3} x_{\text{E--Glc--G6P}}^i 
  + k_2 x_{\text{E}}^i x_{\text{Glc}}^i 
  - k_{-2} x_{\text{E--Glc}}^i
  \\
  \dot x_{\text{E}}^e
  &=
  \beta \left ( x_{\text{E}}^i - x_{\text{E}}^e  \right )
  - k_1 x_{\text{E}}^e x_{\text{Glc}}^e + k_{-1} x_{\text{E--Glc}}^e
  \\
  \dot x_{\text{E}}^i
  &=
  \beta \left ( x_{\text{E}}^e - x_{\text{E}}^i  \right )
  - k_4 x_{\text{E}}^i x_{\text{G6P}}^i + k_{-4} x_{\text{E--G6P}}^i
  \\
  &\phantom{=}
  - k_2 x_{\text{E}}^i x_{\text{Glc}}^i + k_{-2} x_{\text{E--Glc}}^i,
\end{align*}
where $k_1$, $k_{-1}$, $k_2$, $k_{-2}$, $k_3$, $k_{-3}$, $k_4$, $k_{-4}$, $\alpha$, and $\beta$ are the 10 parameters.
Note that this system satisfies \Cref{ass:f_linear_in_theta}.
Following \citet{SchillingsSchwab_GUiY_2015} and \citet{GorbachBauerBuhmann17}, we used this ODE with initial value $x_0 = \mathbbm{1}_M$, time interval $[0., 100.]$ and true parameter $\theta^{\ast}=[0.1,  0.0,  0.4,  0.0,  0.3, 0.0,  0.7,  0.0,  0.1,  0.2]$.
To generate data by \cref{eq:p_z_given_x}, we added Gaussian noise with variance $\sigma^2 = 10^{-5}$ to the corresponding solution at time points $[1., 2., 4., 5., 7., 10., 15., 20., 30., 40., 50., 60., 80., 100.]$. 
The optimizers and samplers were initialized at $\theta^0 = 1.2 \cdot \theta^{\ast} = [0.12, 0, 0.48, 0, 0.36, 0, 0.84, 0, 0.12, 0.24]$, and the forward solutions for all likelihood evaluations were computed with step size $h=0.05$.
To create a good initialization, we accepted the first 30 proposals for PHMC and PLMC. 

\end{document}